\documentclass[10pt]{article}
 
\usepackage{hyperref}
\newcommand{\footremember}[2]{%
    \footnote{#2}
    \newcounter{#1}
    \setcounter{#1}{\value{footnote}}%
}
 
\usepackage{lipsum}

\usepackage{tikz-cd}
\usepackage{amsmath}
\usepackage{amssymb}
\usepackage{amsthm, thmtools}
\usepackage{mathtools}
\usepackage{mathrsfs}

\usepackage{bbm}
\usepackage{bm}
\usepackage{cases}

\usepackage{graphicx}
\graphicspath{{Images/}}
\usepackage[width=.80\textwidth,font=footnotesize]{caption}
\usepackage[letterpaper, margin=1.25in]{geometry}
\usepackage{appendix}
\usepackage{fancyhdr}

\usepackage[shortlabels]{enumitem}
\usepackage[ddmmyyyy]{datetime}
\usepackage{hhline}

\usepackage{listings}
\usepackage{caption}
\usepackage{subcaption}

\newcommand{\fim}{Fisher information matrix}

\renewcommand{\L}{\mathcal{L}}

\newcommand{\N}{\mathbb{N}}

\newcommand{\R}{\mathbb{R}}

\newcommand{\pr}{\mathbb{P}}

\newcommand{\F}{\mathscr{F}}

\DeclarePairedDelimiter\abs{|}{|}
\DeclarePairedDelimiterX{\inner}[2]{\langle}{\rangle}{#1,#2}
\DeclarePairedDelimiterX{\norm}[1]{\|}{\|}{#1}
\DeclarePairedDelimiterX{\normp}[1]{\|}{\|_{TV}}{#1}
\DeclarePairedDelimiterX{\xnorm}[1]{\|}{\|}{#1}

\DeclareMathOperator{\spec}{Spec}

\DeclareMathOperator{\ExpOp}{\mathbb{E}}

\newtheorem{lemma}{Lemma}[section]
\newtheorem*{lemma*}{Lemma}
\newtheorem{prop}{Proposition}[section]
\newtheorem{thm}{Theorem}[section]
\newtheorem*{thm*}{Theorem}

\newtheorem{cor}{Corollary}[section]
\newtheorem*{cor*}{Corollary}

\newtheorem*{ass*}{Assumption}
\theoremstyle{definition}	\newtheorem{Def}{Definition}[section]
\theoremstyle{definition}	\newtheorem{ass}{Assumption}[section]
\theoremstyle{definition}	\newtheorem{rmk}{Remark}[section]

\theoremstyle{definition}

\theoremstyle{plain}	

\theoremstyle{definition}	

\theoremstyle{definition}

\definecolor{OliveGreen}{cmyk}{0.64,0,0.95,0.40}
\definecolor{CadetBlue}{cmyk}{0.62,0.57,0.23,0}
\definecolor{lightlightgray}{gray}{0.93}

\title{Meta-Posterior Consistency for the Bayesian Inference of Metastable Systems}
\author{%
  Zachary P.~Adams\footremember{Berlin-Leipzig}{Freie Universit\"at Berlin, Department of Mathematics and Computer Science and Center for Scalable Data Analytics and Artificial Intelligence, Universit\"at Liepzig}
  \and Sayan Mukherjee\footremember{Leipzig-Durham}{Center for Scalable Data Analytics and Artificial Intelligence and Department of Computer Science, Universit\"at Liepzig and Max Planck Institute for Mathematics in the Natural Sciences and Duke University, Departments of Statistical Science, Mathematics, Computer Science, and Biostatistics \& Bioinformatics}
  }
\date{}

\begin{document}
\maketitle

\begin{abstract}
The vast majority of the literature on learning dynamical systems or stochastic processes from time series has focused on stable or ergodic systems, for both Bayesian and frequentist inference procedures. 
However, most real-world systems are only metastable, that is, the dynamics appear to be stable on some time scale, but are in fact unstable over longer time scales. 
Consistency of inference for metastable systems may not be possible, but one can ask about metaconsistency: 
Do inference procedures converge when observations are taken over a large but finite time interval, but diverge on longer time scales? 
In this paper we introduce, discuss, and quantify metaconsistency in a Bayesian framework. 
We discuss how metaconsistency can be exploited to efficiently infer a model for a sub-system of a larger system, where inference on the global behavior may require much more data, or there is no theoretical guarantee as to the asymptotic success of inference procedures. 
We also discuss the relation between metaconsistency and the spectral properties of the model dynamical system in the case of uniformly ergodic and non-ergodic diffusions. 
\end{abstract}

\section{Introduction}
\label{sec:intro}
The inference of models from data is a central step in many scientific and engineering applications. 
One hopes that, as more data are obtained, a model inferred from these data converges to a model which is the ``best'' representation of the data. 
When this convergence occurs, the inference procedure used is said to be \emph{consistent}. 
Previous studies have shown that, if the system from which data are obtained is \emph{stable}~(for instance, in the sense that it is ergodic), then most reasonable inference procedures will be consistent \cite{IH13,K13}. 

As we usually obtain data from systems which are influenced by external factors, or of which our understanding is incomplete, few real-world systems are truly stable. 
Often, they are only \emph{metastable}, exhibiting behaviour that appears to be stable on some time scale, but unstable on longer time scales. 
For instance, metastable behaviour is recognized as playing a significant role in physiological systems \cite{F04,G20}, population dynamics \cite{ES19,G99}, human and animal microbiomes \cite{C20,M23a,S19}, weather systems~\cite{H08,M06}, and multiagent systems  \cite{carrillo2019aggregation}. 
While metastability has been considered by scientists from various disciplines over the years, there is little to no rigorous mathematical work on the effects of metastability on inference procedures. 
In such metastable systems, we conjecture that inference schemes are \emph{metaconsistent}, in that they appear to converge for a large but finite amount of data, but that this convergence is transient. 

In this paper we introduce, discuss, and precisely quantify metaconsistency in a Bayesian framework. 
We present results on when metaconsistency occurs, and discuss how metaconsistency can be exploited to more efficiently infer a model for a sub-system of a larger system when the amount of data needed to make inferences about the larger system is impractical. 
Finally, in the case of data modelled by continuous time dynamical systems, such as stochastic differential equations (SDEs), we demonstrate the connection between metaconsistency and the spectral theory of dynamical systems. 
More specifically, the main contributions of this paper can be summarized as follows: 
\begin{itemize}
\item A formal introduction of the concept of metaconsistency (Definition \ref{def:mPCR}); 
\item A precise characterization of consistency for ergodic diffusions, with all rates appearing therein related to the spectral gaps and Fisher information of the models, presenting qualitatively known results in as quantitative a fashion as possible (Theorem \ref{thm:ergodic_diffusion}); 
\item A precise characterization of metaconsistency for ergodic and non-ergodic diffusions in a Bayesian framework, all rates appearing therein again being related to the spectral gaps and Fisher information of the models (Theorem \ref{thm:quasiergodic_diffusion}); 
\item A result describing how, in some scenarios, one can speed up learning rates by restricting to each metastable region, learning the dynamics there, and then ``annealing'' what is learned back together (Theorem \ref{thm:cutandsew}). 
\end{itemize}

Before proceeding, we fix the setup and notation of our Bayesian framework in Section \ref{sec:setup}. 
Following this, 
we close the introduction with a short review of the relevant literature in Section \ref{sec:literature}. 
In Section \ref{sec:posterior}~we restate qualitatively known results on the rates of posterior consistency under conditions of \emph{local asymptotic normality}~in as quantitative a fashion as possible that applies to finite time regimes. 
In Section \ref{sec:quasiposterior}~we introduce the concept of metaconsistency and present results allowing for its quantification. 
In Section \ref{sec:ergodicdiffusions}~the timescales of metaconsistency are precisely quantified for metastable diffusion processes. 
Sections \ref{sec:doublewell}~\&~\ref{sec:degeneratedoublewell}~apply our results to study metaconsistency for diffusions in a double well potential, in both cases where the double well is globally or only locally attracting. 
In Section \ref{sec:decomp}~we discuss how, in certain scenarios, one may translate a metastable decomposition of a diffusion process into a linear decomposition of parameter space, and how this decomposition may be exploited to improve posterior contraction rates. 
Appendix \ref{sec:spectral}~contains background on the spectral properties of metastable diffusions.

\subsection{Setup and notation}
\label{sec:setup} 
Let $\Theta\subset\R^p$ be a set of parameters. 
To each $\theta\in\Theta$ we associate a unique probability measure $\pr_\theta$ on an ambient measurable space with filtration $\left(\Omega,\mathcal{F},\left(\mathcal{F}_t\right)_{t\ge0}\right)$. 
Let $E$ be a Banach space. 
For $t>0$, the measure
\[
\pr_\theta^t(\,\cdot\,)\,\coloneqq\,\pr_\theta\left[X_{[0,t]}\in\,\cdot\,\right] 
\]
corresponds to the distribution of an $E$-valued stochastic process $X=(X_s)_{t\ge0}$ over the time interval $[0,t]$, where we write $X_{[0,t]}\coloneqq(X_s)_{s\in[0,t]}$. 
We presume the existence of a \emph{ground truth}~$\theta_0\in\Theta$.

In the terminology of statistics, $\pr^t\coloneqq\{\pr_\theta^t; \, \theta \in \Theta\}$ is a \emph{statistical model}. 
For any $\sigma$-finite measure $\nu$ such that
$\pr_\theta^t\ll\nu$ for each $\theta\in\Theta$ and $t>0$, the \emph{likelihood function} is given by the Radon-Nikodym derivative 
\[
L_t(X \mid \theta)\,\coloneqq\,\frac{d\pr_\theta^t}{d \nu} (X), \qquad \theta\in\Theta, 
\]
where $X=(X_s)_{s\ge0}$ is an observed path. 
In the following, we find it convenient to take $\nu=\pr_{\theta_0}$, though, as a consequence of the Likelihood Principle, one could take any other $\sigma$-finite dominating measure proportional to $\pr_{\theta_0}$ \cite{gonçalves2021definitionlikelihoodfunction}. 

Given a (possibly random) \emph{prior distribution}~$\pi_0$ on $\Theta$, we define the \emph{normalization constant}~as 
\[
\L_t^{\pi_0}(X )\,\coloneqq\,\int_\Theta L_t(X \mid \theta)\,\pi_0(d\theta). 
\]
The \emph{posterior distribution}~is then defined as 
\begin{equation}
\label{eq:pi}
\pi_t(d\theta\mid X)\,\coloneqq\,\frac{1}{\L_t^{\pi_0}(X)} L_t(X \mid \theta)\,\pi_0(d\theta). 
\end{equation}
For $t>0$ we denote by $\pi_t$ the posterior distribution as a map $(F,X)\mapsto\pi_t(F\mid X)$. 

\begin{Def}
\label{def:consistent}
The statistical model $\pr^t$ is \emph{consistent}, or exhibits \emph{posterior consistency}~(at $\theta_0$), when there exists a family of measurable subsets $(F_t)_{t\ge0}$ of $\Theta$ such that $F_t\rightarrow\{\theta_0\}$ as $t\rightarrow\infty$ and, for any $\eta>0$, we have 
\begin{equation}
\pr_{\theta_0}\left[\pi_t(F_t)<1-\eta\right]\,\xrightarrow[t\rightarrow\infty]{}\,0. 
\end{equation}
If $\pr^t$ is consistent at $\theta_0$, we say that it has a \emph{posterior contraction rate}~(PCR) of $(\epsilon_t)_{t\ge0}$ and a \emph{posterior contraction certainty rate}~(PCCR) of $(\eta_t)_{t\ge0}$ if $\epsilon_t\rightarrow0$ and $\eta_t\rightarrow0$ as $t\rightarrow\infty$, and 
\begin{equation}
\label{eq:PCRPCCR}
\pr_{\theta_0}\left[\pi_t\left(B_{\epsilon_t}(\theta_0)\right)<1-\eta_t\right]\,\le\,\eta_t, 
\end{equation} 
where $B_\epsilon(\theta_0)\subset\Theta$ denotes the ball of radius $\epsilon$ centered at $\theta_0$. 
\end{Def}

We emphasize that most studies only focus on the PCR of Bayesian schemes, ignoring the PCCR. 
Note that it is entirely possible that the ``Markov type'' bound in \eqref{eq:PCRPCCR}~may not be optimal. 
When $\eta_t$ is obtained such that $\pr_{\theta_0}\left[\pi_t(B_{\epsilon_t}(\theta_0))<1-\eta_t\right]\rightarrow0$ as $t\rightarrow0$, one can likely get a better bound than just $\eta_t$ on the right hand side of \eqref{eq:PCRPCCR}, perhaps a ``DKW type'' bound. 
For simplicity, we define the PCCR as the quantity such that \eqref{eq:PCRPCCR}~as written holds, rather than looking for optimal right hand side bounds. 

We now formalize the concept of \emph{metaconsistency}. 
While the definition remains somewhat imprecise due to lack of specification of the time interval appearing in the definition, we obtain quantitative characterizations of metaconsistent timescales in the following sections. 

\begin{Def}
\label{def:mPCR}
A statistical model $\pr^t$ is \emph{metaconsistent}, or \emph{meta-posterior consistent}~(at $\theta_0$), if there exist $(\hat{\epsilon}_t)_{t\ge0},\,(\hat{\eta}_t)_{t\ge0}\in[0,\infty)$ such that $\hat{\epsilon}_t\rightarrow0$ as $t\rightarrow\infty$ and $\hat{\eta}_t\ll1$ for all $t$ in some finite time interval $[\underline{t},\overline{t}]$, and 
\[
\pr_{\theta_0}\left[\pi_t(B_{\hat{\epsilon}_t}(\theta_0))<1-\hat{\eta}_t\right]\,\le\,\hat{\eta}_t. 
\]
We refer to $(\hat{\epsilon}_t)_{t\ge0}$ and $(\hat{\eta}_t)_{t\ge0}$ as the \emph{meta-posterior contraction rate}~(m-PCR) and \emph{meta-posterior contraction certainty rate}~(m-PCCR), respectively. 
\end{Def}

The concept of ``local asymptotic normality'' is often used in the analysis of inference procedures, and will be used throughout the paper. 
For a linear space $U$ and $x\in U$, let $U-x=\{u-x:u\in U\}$. 

\begin{Def}
\label{def:LAN}
A statistical model $\pr^t$ is \emph{Locally Asymptotically Normal}~(LAN) at $\theta_0\in\Theta$ if there exist 
\begin{enumerate}[(i)]
\item a family of linear maps $\phi=(\phi_t)$, $\phi_t:\R^p\rightarrow\Theta-\theta_0$, such that $\phi_t\rightarrow0$ as $t\rightarrow\infty$, 
\item a sequence of random variables $\Delta=(\Delta_t)_{t\ge0}$ such that $\Delta_t=\Delta_t(X)$ converges in $\pr_{\theta_0}$-mean square to a standard
\footnote{
As discussed further in Section \ref{sec:ergodicdiffusions}, many models satisfy a similar condition, but where $\Delta_\infty$ has a nontrivial covariance matrix. 
See for instance \cite{O15}. 
This condition is often referred to as a \emph{locally asymptotic mixed normality (LAMN)}. 
The interested reader may trace our proofs, and note that allowing for the LAMN condition does not significantly change their logic. 
However, the notation and computations in Section \ref{sec:ergodicdiffusions}~are greatly simplified when LAN is assumed as opposed to LAMN, so we restrict ourselves to this setting throughout this paper. 
}~
Gaussian $\Delta_\infty$ as $t\rightarrow\infty$, and 
\item a family of random maps $r=(r_t)$, $r_t:\Theta\times\R^p\rightarrow\Theta-\theta_0$, such that $r_t\rightarrow0$ as $t\rightarrow\infty$ in $\pi_0\otimes\pr_{\theta_0}$-mean, 
\end{enumerate}
such that the equality 
\[
\frac{d\pr_{\theta_0+\phi_tu}^t}{d\pr_{\theta_0}^t}(X)\,=\,\exp\left(\Delta_t\cdot u-\norm*{u}^2 + r_t(\theta_0,u)\right) 
\]
holds $\pr_{\theta_0}$-almost surely. 
When we want to emphasize that a model $\pr^t$ is LAN for a specific choice of $\phi,\,\Delta,\,r$, we say that it is \emph{$(\phi,\Delta,r)$-LAN}~at $\theta_0$. 
\end{Def}

\begin{rmk}
As $\Theta$ is assumed to be finite dimensional, we restrict ourselves to the parametric setting. 
However, we expect many of the results below to translate to the nonparametric setting with suitable modifications to the decay rates appearing in our main results. 
It would also be interesting to extend the theory developed here to the case where $\Theta$ is intrinsically nonlinear.  
\end{rmk}

\subsection{Literature and context}
\label{sec:literature}
To the authors' knowledge, the earliest discussions of posterior consistency are found in Doob \cite{key33460m}~and Schwartz \cite{Schwartz1965}. 
Textbooks providing surveys of the field include Ghosal \&~van der Vaart \cite{GVDV17}~and Ibragimov \&~Khasminskii \cite{IH13}. 
More recent results on posterior consistency and convergence rates in more general settings, including non-i.i.d.~observations and nonparametric models, are presented in \cite{GGS95,G00,GV07,GVDV17,VDVVZ08,VDVVZ08b}. 

Consistency of statistical models comprised of stochastic dynamical systems 
has mainly -- perhaps necessarily -- been restricted to the ergodic setting, as in \cite{I62,MMN22,SM21}. 
In the parametric setting, when a statistical model is comprised of ergodic diffusions, the statistical model is LAN \cite{LeCam:60,VDV00}. 
In this case, one has a PCR of order $t^{-1/2}$. 
See \cite{K13,PK83,Y92}~for early work on parametric models of continuously or discretely observed diffusions. 
In the case of nonparametric statistical models consisting of ergodic diffusions, one expects a PCR on the order of $t^{-\alpha}$, where $\alpha<1/2$ depends on the regularity of the model considered and the choice of prior distribution. 
This is seen in \cite{G23,GR22,PVZ09,P13}, for nonparametric models of continuously sampled diffusions, and in \cite{GS14,VDMS17,VDMVZ13,VZ13}, for nonparametric models of discretely sampled diffusions. 
The fact that one has consistency for nonparametric statistical models comprised of ergodic diffusions is striking, due to well known counter-examples to consistency in the nonparametric setting \cite{DF86}. 
Further extensions of the study of posterior consistency have been considered by Ogihara \cite{O23}~for jump-diffusion processes, and in Lysy \&~Pillai \cite{LP13}~for SDEs driven by fractional Brownian motion. 

Closely related to the LAN property of statistical models is the Bernstein von Mises Theorem, which is perhaps the most commonly used tool in proving posterior consistency \cite{GVDV17}. 
The Bernstein von Mises Theorem establishes that, under some conditions, a Laplace approximation of the posterior distribution is asymptotically perfect as the amount of data from which the posterior is constructed increases \cite{10.1214/aos/1176347751}. 
Other recent work related to inference from stationary ergodic processes includes denoising/filtering \cite{Lalley1999,LalleyNobel2006}, 
consistency of maximum likelihood estimation \cite{McGoff2015}, 
forecasting and density estimation \cite{Hang2017, Hang2018, Steinwart2009}, 
empirical risk minimization \cite{mcgoff2016variational,mcgoff2020empirical}, inference on time series from LAN conditions \cite{Francq_Zakoian_2023}, and data assimilation and uncertainty quantification \cite{Baek_2023,cdrs2009,Law2015,stuart2010}.
For further discussion, see the survey \cite{McGoffSurvey2015}. 

The above cited research focuses on the \emph{asymptotic}~behaviour of Bayesian posteriors. 
As metaconsistency is a fundamentally finite time phenomenon, our aim is to study the concentration properties of posteriors that hold \emph{in the finite time regime}, rather than just asymptotically. 
In the course of studying these concentration properties, we find assuming the prior $\pi_0$ of our Bayesian setup to be random facilitates our proofs. 
The assumptions on the randomness of $\pi_0$ which we make are not needed in the above mentioned studies, as they focus on \emph{asymptotic}~consistency. 
We argue that the assumptions we make on the randomness of $\pi_0$ are in fact in agreement with Bayesian methodology and applied Bayesian analysis. 
Indeed, random priors arise as ``two level'' priors in what are called
\emph{Bayesian hierarchical model}s, which have been known for over fifty years \cite{LindleySmith72}. 
Bayesian hierarchical models are bread-and-butter tools in Bayesian applications, see for instance \cite{Chada_2018,Clark,Gelman_Hill_2006,PoleWestHarrison}. 

We also emphasize that, while the order of the PCR has been the subject of much analysis over the past few decades, the PCCR has not been much studied. 
As seen below, characterizing the PCCR \emph{in addition}~to the PCR is essential in the analysis of posterior consistency for statistical models consisting of metastable diffusions. 
It is in the PCCR that the effects of metastability are most notable, as the PCCR is directly related to the spectral gap of the observed process (see Theorem \ref{thm:ergodic_diffusion}~below). 

The spectral properties of a stochastic process and the phenomenon of metastability are known to be deeply connected \cite{BDH16}. 
The connection between metastability and spectral theory of stochastic processes in low temperature regimes was established in the 70s, see Ventcel \cite{V76,V72}, Freidlin \cite{F73}, Devinatz, Ellis, \&~Friedman \cite{D74}, and later, Mathieu \cite{M95}. 
Prior to this, quantitative results on metastability (of a form similar to those obtained using spectral theory) were first presented in the context of physical chemistry, in the form of what is today referred to as ``Arrhenius Law'' \cite{A89}. 
This was then refined into the celebrated Eyring-Kramers formula by Eyring \cite{E35}~and Kramers \cite{K40}~independently. 
As an alternative to spectral theory, large deviation theory was developed to handle metastability in non-reversible processes -- see the books of Freidlin-Wentzell \cite{FW} or Dembo-Zeitouni \cite{DZ}~for further discussion. 
In 2005, the Eyring-Kramers formula was sharpened in the setting of diffusions with multi-well potentials and additive noise by Bovier, Gayrard, \&~Klein \cite{BGK05,BGK05ii}~using the concept of ``capacitances''. 
The sharp asymptotic formulas of \cite{BGK05}~were then connected to functional inequalities in \cite{MS14}~by the decomposition of energy functionals. 
Results on metastability for diffusions with non-gradient drift or in the infinite dimensional setting are scarce, but see \cite{A24,BG13,M23,VW24}. 
For a review of Eyring-Kramers formulas for diffusions with gradient drift, see \cite{B13}. 

Let us close this incomplete literature review by commenting on some of the applications of the theory developed here. 
In addition to the applied probability and mathematical statistics question of convergence of estimators in the metastable setting, our results provide a theoretical basis for applications of machine learning and time series analysis \cite{JMLR:v23:19-882,West1989}. 
An application motivating the theory in our paper is microbial ecology, where metastability is increasingly being recognized as an important factor \cite{C20,M23a,S19}. 
Modern sequencing technology allows for the collection of time series of microbiome data, a vector of the composition of the microbial community in the gut, skin, or soil over time. 
An interesting aspect of microbial ecology is that both classical time series models as well deterministic models from dynamical systems such as the Lotka-Volterra model are used extensively. 
Co-author Mukherjee has worked extensively on microbial time series including the development of
novel methodologies for use by microbial ecologists, computational and statistical approaches that scale to high-dimensional and unevenly sampled time series, and the analysis of microbial data \cite{JMLR:v23:19-882,Silverman2020NZ,Washburne2019}.
We know microbiome time series are not ergodic or independent and can look stable at some time scale but then can become unstable, so a realistic rigorous theory for these methods is lacking.

\section{Posterior consistency from LAN conditions}
\label{sec:posterior}
\subsection{Statement of main result}
We here discuss how one can obtain precise estimates on the PCR and PCCR from conditions of local asymptotic normality (LAN). 
LAN conditions are well known to be sufficient for posterior consistency. 
The contribution of this section is the detail with which we present estimates on PCRs and PCCRs from LAN conditions, facilitating the comparison of the true values with the metastable values in the sections to follow. 
Furthermore, the theory outlined here allows us to explicitly compute all bounds on the PCR and PCCR for a class of metastable diffusions in Section \ref{sec:doublewell}~below. 

\begin{ass}
\label{ass:LAN}
Fix a (random) prior distribution $\pi_0$ on $\Theta$ with $\pr_{\theta_0}$-almost surely finite moments. 
The family $(\pr_\theta)_{\theta\in\Theta}$ is $(\phi,\Delta,r)$-LAN at some fixed $\theta_0$ which is almost surely contained in the support of $\pi_0$. 
Denoting by $\abs*{\phi_t}$ the smallest singular value of $\phi_t$, $\abs*{\phi_t}\rightarrow0$ and $\abs*{\phi_t}^{-\alpha}\phi_t\rightarrow0$ as $t\rightarrow\infty$ for $\alpha\in(0,1)$. 
Defining 
\begin{equation}
\label{eq:mtm}
m_t(du)\coloneqq\exp\left(-\norm*{\frac{1}{2}\Delta_t-u}^2+r_t\right)\,\pi_0(du), 
\end{equation}
there exists a deterministic $\underline{m}>0$ such that $\pr_{\theta_0}$-almost surely, $m_t(\Theta)\ge\underline{m}$ for all $t>0$. 
\end{ass}


Under the LAN condition, the PCR is determined by the rate at which $\phi_t$ goes to zero. 
Moreover, we can characterize the PCCR in terms of the rates at which $\Delta_t$ converges to $\Delta_\infty$ (in mean square), and at which $r_t$ converges to zero (in mean). 
To state this more precisely, 
define \begin{equation}
\label{eq:epsilons}
\varepsilon_\Delta(t)\,\coloneqq\,\ExpOp_{\theta_0}\left[\norm*{\frac{1}{2}\left(\Delta_t-\Delta_\infty\right)}^2\right], \qquad\varepsilon_r(t)\,\coloneqq\,\ExpOp_{\theta_0}\left[\int_\Theta\abs*{r_t}\,\pi_0(d\theta)\right],\qquad t\ge0.   
\end{equation}
For fixed $\alpha\in(0,1)$ and $\delta>0$, define 
\begin{equation}
\label{eq:Ft}
F_t\,\coloneqq\,\abs*{\phi_t}^\alpha B_\delta(\theta_0),\qquad U_t\,\coloneqq\,\abs*{\phi_t}^\alpha\phi_t^{-1}(B_\delta(0)),\qquad t\ge0. 
\end{equation}
By Assumption \ref{ass:LAN}, $F_t\rightarrow\{\theta_0\}$ and $U_t\rightarrow\Theta$ as $t\rightarrow\infty$, while both $\varepsilon_\Delta(t),\varepsilon_r(t)\rightarrow0$ as $t\rightarrow\infty$. 
It will be seen that the posterior distribution concentrates on $F_t$ as $t\rightarrow\infty$, so that the PCR can be characterized by $\abs*{\phi_t}$, while the PCCR can be characterized by the rates at which $\Delta_t\rightarrow\Delta_\infty$, $r_t\rightarrow0$, and $\pi_0(U_t^c)\rightarrow0$, in various senses. 

\begin{prop}
\label{thm:posteriorconsistency}
Under Assumption \ref{ass:LAN}, the model $\pr^t$ is consistent at $\theta_0$ with a PCR determined by $\abs*{\phi_t}$, and PCCR determined by $\varepsilon_\Delta,\varepsilon_r$. 
In particular, let $(F_t)_{t\ge0}$ and $(U_t)_{t\ge0}$ be the sequences of subsets of $\Theta$ defined in \eqref{eq:Ft}, and let $\varepsilon_\Delta$ and $\varepsilon_r$ be as in \eqref{eq:epsilons}. 
Then, there exists a constant $c>0$ depending solely on $\pi_0$ such that, defining 
\begin{equation}
\label{eq:H1}
H(t)\,\coloneqq\,\underline{m}^{-1}\left(\varepsilon_\Delta(t) + c\varepsilon_\Delta(t)^{1/2}+\varepsilon_r(t) + \ExpOp_{\theta_0}\left[\pi_0(U_t^c)\right]\right), 
\end{equation}
$F_t\rightarrow\{\theta_0\}$ and $H(t)\rightarrow0$ as $t\rightarrow\infty$, and for $t\ge0$, 
\begin{equation}
\label{eq:posteriorconsistency}
\pr_{\theta_0}\left[\pi_t(F_t \mid X)<1-H(t)^{1/2}\right]\,\le\,H(t)^{1/2}. 
\end{equation}
\end{prop}

\begin{rmk}
Requiring $m_t(\Theta)$ to be bounded away from zero uniformly in $t>0$ and $\pr_{\theta_0}$-almost surely, as in Assumption \ref{ass:LAN}, is essentially requiring that each realization of $\pi_0$ and the density $\exp\left(-\norm*{\frac{1}{2}\Delta_t-u}^2+r_t\right)$ are concentrated on regions of parameter space that have sufficient overlap. 
%
Previous studies have obtained posterior consistency with a PCR of order $t^{-1/2}$ without making such an assumption, see for instance \cite{IH13,K13}. 
However, we suggest that this is due to the fact that these studies are only concerned with \emph{asymptotic}~consistency, whereas the purpose of this study is to shorten the bridge between theory and practice, obtaining bounds that hold in the finite time regime as well. 
We also note that a random prior is in agreement with the practice of \emph{hierarchical Bayes}, which is a popular approach to Bayesian inference in the applied statistics community. 
One motivation for hierarchical Bayesian models or multilevel models is to allow for the dependence of the joint probability model over parameters. 
In our case for example the means of the random priors. 
The idea of different levels of probability goes back at least to the 1950's, see for instance \cite{Good59}.
\end{rmk}

\begin{rmk}
\label{rmk:rates}
In applications, such as the case of ergodic diffusions considered in Section \ref{sec:ergodicdiffusions}, below, one expects $\abs*{\phi_t}$, $\varepsilon_\Delta(t)$, and $\varepsilon_r(t)$ to decay at algebraic rates $t^{-\alpha}$ for some $\alpha\in(0,1/2]$, as discussed in Section \ref{sec:literature}, where $\alpha=1/2$ in the parametric setting. 
However, in many cases more precise information than this can be obtained. 
As seen in Section \ref{sec:ergodicdiffusions}, 
\[
\abs*{\phi_t}\,\simeq\,s_1^{-1}t^{-1/2}\quad\text{ and }\quad \varepsilon_\delta(t),\varepsilon_r(t)\,\simeq\,(1-e^{-\gamma t})\gamma^{-1}t^{-1}+s_1^{-1}t^{-1}, 
\]
where $s_1$ is the Fisher information (to be defined) and $\gamma$ is the spectral gap of the process $(X_t)_{t\ge0}$ under $\pr_{\theta_0}$. 
Consequently, when restricting attention to a metastable state increases the Fisher information, the PCR is improved (as in the example of Section \ref{sec:doublewell}, while if doing so increases the spectral gap of the true process, the PCCR is improved (as in the example of Section \ref{sec:degeneratedoublewell}). 
As we discuss in Appendix \ref{sec:spectral}, restricting a metastable system to a single metastable state greatly increases the spectral gap of the system. 
\end{rmk}

\subsection{Proof of Proposition \ref{thm:posteriorconsistency}}
Before proving Proposition \ref{thm:posteriorconsistency}, we require the following lemmas. 

\begin{lemma}
\label{lemma:mtconvergence1}
Let $(m_t)_{t\ge0}$ be a sequence of (possibly random) measures on a measurable space $(E,\F)$ converging in total variation to a nontrivial regular measure $m$. 
Letting $(U_t)_{t\ge0}\subset\F$ be a nested sequence of measurable sets such that $U=\bigcup_{t\ge0}U_t$ and assuming that $\inf_{t\ge0}m_t(U)\eqqcolon\underline{m}>0$, 
\[
\abs*{\frac{m_t(U_t)}{m_t(U)}-1}\,\le\,\underline{m}^{-1}\left(2\norm*{m_t-m}_{TV}+\abs*{m(U_t)-m(U)}\right)\,\xrightarrow[t\rightarrow\infty]{}\,0. 
\]
\begin{proof}
An application of the triangle inequality yields the result, 
\[
\begin{aligned}
\abs*{\frac{m_t(U_t)}{m_t(U)}-1}\,&\le\,\abs*{\frac{m_t(U_t)-m(U_t)}{m_t(U)}} + \abs*{\frac{m(U_t) - m(U)}{m_t(U)}} + \abs*{\frac{m(U) -m_t(U)}{m_t(U)}}\\
&\le\,\underline{m}^{-1}\left( 2\norm*{m_t-m}_{TV} + \abs*{m(U_t)-m(U)} \right). 
\end{aligned}
\]
\end{proof}
\end{lemma}

In the setting of a $(\phi,\Delta,r)$-LAN statistical model, we take $m_t$ as in \eqref{eq:mtm}, 
\[
m_t(du)\coloneqq\exp\left(-\norm*{\frac{1}{2}\Delta_t-u}^2+r_t\right)\,\pi_0(du), 
\]
and similarly define 
\begin{equation}
\label{eq:m}
m(du)\coloneqq \exp\left(-\norm*{\frac{1}{2}\Delta_\infty-u}^2\right)\,\pi_0(du).  
\end{equation}
We now obtain a result on the rate of convergence of $\norm*{m_t-m}_{TV}$ to zero as $t\rightarrow\infty$. 
\begin{lemma}
\label{lemma:mtconvergence2}
Let Assumption \ref{ass:LAN}~hold, so that $(\pr_\theta)_{\theta\in\Theta}$ is $(\phi,\Delta,r)$-LAN at $\theta_0\in\Theta$.
Define $m_t$ as in \eqref{eq:mtm}, $m$ as in \eqref{eq:m}, and $\varepsilon_\Delta(t)$ and $\varepsilon_r(t)$ as in \eqref{eq:epsilons}. 
Then, 
\[
\ExpOp_{\theta_0}\left[\norm*{m_t-m}_{TV}\right]\,\le\,\varepsilon_\Delta(t) + c\varepsilon_\Delta(t)^{1/2}+\varepsilon_r(t), 
\]
where $c>0$ is a constant depending solely on the prior $\pi_0$, and $\varepsilon_\Delta(t),\,\varepsilon_r(t)$, tend to zero as $t\rightarrow\infty$. 
\begin{proof}
Remark that, from the reverse triangle inequality, 
\[
\norm*{\frac{1}{2}\Delta_t-u}^2\,\ge\,\norm*{\frac{1}{2}\left(\Delta_t-\Delta_\infty\right)}^2+2\norm*{\frac{1}{2}\left(\Delta_t-\Delta_\infty\right)}\norm*{\frac{1}{2}\Delta_\infty-u} + \norm*{\frac{1}{2}\Delta_\infty-u}^2. 
\]
It then follows that, as the exponential has Lipschitz constant one on the negative half line, 
\[
\begin{aligned}
&2\ExpOp_{\theta_0}\left[\norm*{m_t-m}_{TV}\right]\\
&\qquad=\,\ExpOp_{\theta_0}\left[\int_\Theta\abs*{e^{-\norm*{\frac{1}{2}\Delta_t-u}^2+r_t}-e^{-\norm*{\frac{1}{2}\Delta_\infty-u}^2}}\,\pi_0(du)\right]\\
&\qquad\le\,\ExpOp_{\theta_0}\left[\int_\Theta e^{-\norm*{\frac{1}{2}\Delta_t-u}^2}\abs*{ e^{-\norm*{\frac{1}{2}\left(\Delta_t-\Delta_\infty\right)}^2-\norm*{\frac{1}{2}\left(\Delta_t-\Delta_\infty\right)}\norm*{\frac{1}{2}\Delta_\infty-u}+r_t}-1 } \,\pi_0(du) \right]\\
&\qquad\le\,\ExpOp_{\theta_0}\left[\int_\Theta\norm*{\frac{1}{2}\left(\Delta_t-\Delta_\infty\right)}^2+\norm*{\frac{1}{2}\left(\Delta_t-\Delta_\infty\right)}\norm*{\frac{1}{2}\Delta_\infty-u}+\abs*{r_t}\,\pi_0(du)\right]\\
&\qquad\le\,\ExpOp_{\theta_0}\left[\norm*{\frac{1}{2}\left(\Delta_t-\Delta_\infty\right)}^2\right] + 2\ExpOp_{\theta_0}\left[\norm*{\frac{1}{2}\left(\Delta_t-\Delta_\infty\right)}^2\right]^{1/2}\ExpOp_{\theta_0}\left[\int_\Theta\norm*{\frac{1}{2}\Delta_\infty-u}\,\pi_0(du)^2\right]^{1/2}\\
&\qquad\qquad\qquad + \ExpOp_{\theta_0}\left[\int_\Theta\abs*{r_t}\,\pi_0(du)\right]\\
&\qquad\eqqcolon\,\varepsilon_\Delta(t) + c\varepsilon_\Delta(t)^{1/2} + \varepsilon_r(t),  
\end{aligned}
\]
where we remark that the constant $c>0$ in the last line above only depends on the prior $\pi_0$. 
\end{proof}
\end{lemma}

We are now prepared to prove posterior consistency of a statistical model from LAN conditions, paying attention to rates of contraction and contraction certainty. 

\begin{proof}[Proof of Proposition \ref{thm:posteriorconsistency}]
Let $m_t$ and $m$ be as defined in \eqref{eq:mtm}, and for measurable $F\subset\Theta$ let $U_t(F)\coloneqq\phi_t^{-1}(F-\theta_0)$. 
Note that in our setup, $U_t(\Theta)=\Theta$. 
Due to the LAN property of $(\pr_\theta)_{\theta\in\Theta}$, 
\[
\pi_t(F \mid X)\,=\,\frac{\int_{U_t(F)} e^{-\norm*{\frac{1}{2}\Delta_t-u}^2+r_t}\,\pi_0(du)}{\int_\Theta e^{\norm*{\frac{1}{2}\Delta_t-u}^2+r_t}\,\pi_0(du)}\,=\,\frac{m_t(U_t(F))}{m_t(\Theta)}. 
\]
So long as $\theta_0\in F$, it follows that $U_t(F)\rightarrow\Theta$ as $t\rightarrow\infty$ (otherwise $U_t(F)\rightarrow\varnothing$). 
Hence, 
noting that $m(U)\le\pi_0(U)$ for all measurable $U\subset\Theta$ almost surely, and applying Lemmas \ref{lemma:mtconvergence1}~\&~\ref{lemma:mtconvergence2}~yields  
\[
\begin{aligned}
\ExpOp_{\theta_0}\left[\abs*{\pi_t(F \mid X)-1}\right]\,&\le\, \ExpOp_{\theta_0}\left[\underline{m}^{-1}\left(2\norm*{m_t-m}_{TV}+m(U_t(F)^c)\right)\right]\\ 
&\le\,\underline{m}^{-1}\left(\varepsilon_\Delta(t) + c\varepsilon_\Delta(t)^{1/2}+c\varepsilon_r(t) + \ExpOp_{\theta_0}\left[\pi_0(U_t(F)^c)\right]\right). 
\end{aligned}
\]
Consequently, by Markov's inequality we have that, for any $\eta>0$, 
\[
\pr_{\theta_0}\left[\pi_t(F \mid X)<1-\eta\right]\,\le\,\eta^{-1}\underline{m}^{-1}\left(\varepsilon_\Delta(t) + c\varepsilon_\Delta(t)^{1/2}+c\varepsilon_r(t) + \ExpOp_{\theta_0}\left[\pi_0(U_t(F)^c)\right]\right). 
\]
Taking at each point in time $t>0$ the set $F_t=\abs*{\phi_t}^\alpha B_\delta(\theta_0)$, 
\[
U_t(F_t)\,=\,\phi_t^{-1}(\abs*{\phi_t}^\alpha B_\delta(\theta_0)-\theta_0)\,=\,\abs*{\phi_t}^\alpha\phi_t^{-1}B_\delta(0), 
\]
which tends to the whole parameter space $\Theta$ as $t\rightarrow\infty$ by Assumption \ref{ass:LAN}. 
By the regularity of $\pi_0$, this completes the proof. 
\end{proof}

In Section \ref{sec:ergodicdiffusions}, we compute $\varepsilon_\Delta(t)$ and $\varepsilon_r(t)$ explicitly for the example of ergodic diffusions. 
In this case, it will be seen that $\varepsilon_\Delta$ and $\varepsilon_r$ tend to zero at rates determined by the spectral gap of the ergodic process and Fisher information of the statistical model in question. 
We then explicitly compute these constants for two concrete examples in Sections \ref{sec:doublewell}~\&~\ref{sec:degeneratedoublewell}.

\section{Metaconsistency}
\label{sec:quasiposterior}
\subsection{Statement of main result}
In this section, we do not require $(\pr_\theta)_{\theta\in\Theta}$ to be LAN. 
Instead, we make the following assumption, presupposing the existence of \emph{another}~family $(\hat{\pr}_\theta)_{\theta\in\Theta}$, such that the distributions of $(X_t)_{t\ge0}$ under $(\pr_\theta)_{\theta\in\Theta}$ and $(\hat{\pr}_\theta)_{\theta\in\Theta}$ are equal up to a stopping time $\tau$, and such that $(\hat{\pr}_\theta)_{\theta\in\Theta}$ \emph{is}~LAN. 
One could think of the following discussion as a backwards coupling argument, similar to the decouplings studied in \cite{E18,JM17,V16}~-- though unlike in those works, here we take the decoupling time $\tau$ as known, and construct a process that is perfectly coupled up to $\tau$. 
In the metastable setting, $\tau$ is the transition time from one of the metastable subsystems of the true system, and $(\hat{\pr}_\theta)_{\theta\in\Theta}$ is a modified model such that the metastable subsystem is globally stable. 
See Section \ref{sec:ergodicdiffusions}~below for concrete examples. 

\begin{ass}
\label{ass:quasi}
There exists a stopping time $\tau$ and a family of probability measures $(\hat{\pr}_\theta)_{\theta\in\Theta}$ such that, for a fixed (random) prior distribution $\pi_0$ on $\Theta$, for $\pi_0$-almost all $\theta\in\Theta$ and $t>0$, 
\begin{enumerate}[(i)]
\item $\pr_\theta[t<\tau]>0$ and $\pr_\theta[\tau<\infty]=1$. 
\item The distributions of $(X_t)_{t\ge0}$ under $\pr_\theta$ and $\hat{\pr}_\theta$ are equal on the event $\{t<\tau\}$, \emph{i.e.}~
\begin{equation}
\label{eq:modpr1}
\hat{\pr}_\theta\left[X_{[0,t]}\in\,\cdot\,,t<\tau\right]\,=\,\pr_\theta\left[X_{[0,t]}\in\,\cdot\,,t<\tau\right]\quad\text{ for each $\theta\in\Theta$ and $t>0$}.  
\end{equation}
\item $\hat{\pr}_\theta\left[t<\tau\right]\,=\,\pr_\theta\left[t<\tau\right]$ for each $\theta\in\Theta$ and $t>0$. 
\item The family $(\hat{\pr}_\theta)_{\theta\in\Theta}$ is $(\hat{\phi},\hat{\Delta},\hat{r})$-LAN at some specific $\theta_0\in\Theta$, and satisfies Assumption \ref{ass:LAN}. 
\end{enumerate}
\end{ass}

As in Section \ref{sec:posterior}, we require some auxiliary notation to state the main result of this section. 
For fixed $\alpha\in(0,1)$, $\delta>0$, and variable $t>0$, define the sets 
\begin{equation}
\label{eq:hatF}
\hat{F}_t\,\coloneqq\,\abs{\hat{\phi}_t}^\alpha B_\delta(\theta_0),\qquad \hat{U}_t\,\coloneqq\,\abs{\hat{\phi}_t}^\alpha\hat{\phi}_t^{-1}(B_\delta(0)). 
\end{equation}
It is shown below that the m-PCR can be characterized by the decay rate of $\hat{\phi}_t$. 
This is due to the fact that, in the metastable setting, the posterior \emph{appears}~to concentrate on the set $\hat{F}_t$ defined in \eqref{eq:hatF}~up to the stopping time $\tau$. 
Moreover, define $\varepsilon_{\hat{\Delta}}$ and $\varepsilon_{\hat{r}}$ as in \eqref{eq:epsilons}, and $\hat{\underline{m}}\coloneqq\,\inf_{t\ge0}\hat{m}_t(\Theta)$, where $\hat{m}_t$ is the measure defined in \eqref{eq:mtm}, but with $\Delta_t$ and $r_t$ replaced by $\hat{\Delta}_t$ and $\hat{r}_t$, and expectations being taken with respect to $\hat{\pr}_{\theta_0}$. 
 
In analogy with Section \ref{sec:posterior}, define 
\begin{equation}
\label{eq:hatH}
\hat{H}(t)\,\coloneqq\,\hat{\underline{m}}^{-1}\left(\varepsilon_{\hat{\Delta}}(t) + c\varepsilon_{\hat{\Delta}}(t)^{1/2}+\varepsilon_{\hat{r}}(t) + \ExpOp_{\theta_0}\left[\pi_0(\hat{U}_t^c)\right]\right), 
\end{equation}
which tends to zero under Assumption \ref{ass:quasi}~by Proposition \ref{thm:posteriorconsistency}. 
The main result of this section describes how the posterior consistency of $\hat{\pi}_t$ relates to the behaviour of $\pi_t$. 
We can now state the main result of this section. 

\begin{thm}
\label{thm:quasiposteriorconsistency}
Let Assumption \ref{ass:quasi}~hold, and let $\hat{H}(t)$ be as defined in \eqref{eq:hatH}, and $\hat{F}_t,\,\hat{U}_t,$ as in \eqref{eq:hatF}~for some fixed $\delta>0$ and $\alpha\in(0,1)$. 
Then, for $t>0$, 
\begin{equation}
\label{eq:quasiposteriorconsistency}
\pr_{\theta_0}\left[\pi_t(\hat{F}_t \mid 
X)<1-\hat{H}(t)^{1/2}\right]\,\le\,\hat{H}(t)^{1/2} + \pr_{\theta_0}\left[t>\tau\right]. 
\end{equation}
\end{thm}

Before proceeding with the proof of Theorem \ref{thm:quasiposteriorconsistency}, let us make a few remarks. 
Suppose that \emph{both}~Assumptions \ref{ass:LAN}~\&~\ref{ass:quasi}~hold, so that the estimates in \eqref{eq:posteriorconsistency}~\&~\eqref{eq:quasiposteriorconsistency}~both hold. 
If, additionally, 
\begin{enumerate}
\item $\hat{H}(t)$ decays to zero much faster than $H(t)$,
\item $\hat{F}_t$ contracts at a rate at least as fast as $F_t$, and
\item $\pr_{\theta_0}[t>\tau]$ remains small on a time scale where $\hat{H}(t)$ is small, 
\end{enumerate}
then we observe meta-posterior consistency, in the sense of Definition \ref{def:mPCR}. 
In Section \ref{sec:ergodicdiffusions}~below, we consider the case where $X$ is an ergodic diffusion under $\pr_{\theta_0}$. 
We see that the PCCRs $H(t)$ and $\hat{H}(t)$ are determined by both the spectral gap and Fisher information of $X$ under $\pr_{\theta_0}$ and $\hat{\pr}_{\theta_0}$, respectively, while the contraction rates of $F_t$ and $\hat{F}_t$ are determined purely by the Fisher information of $X$ under $\pr_{\theta_0}$ and $\hat{\pr}_{\theta_0}$. 
When the ergodic diffusion in consideration is metastable, it is known that restricting attention (in a sense to be made precise) to a single metastable state increases the spectral gap \cite{MS14}. 
On the other hand, restricting attention to a single metastable state may, as shown in the examples of Sections \ref{sec:doublewell}~\&~\ref{sec:degeneratedoublewell}, either increase or decrease the Fisher information of the system. 
When the Fisher information decreases upon restricting attention to a metastable state, how much it does so depends on how much less the metastable state ``sees'' the parameter to be estimated. 
In the example of Section \ref{sec:doublewell}, the restriction does not cost us as much information, but does greatly increase the spectral gap of the system we consider, so that the approach seems worth it. 

Counter-intuitively, in the example of Section \ref{sec:degeneratedoublewell}, we see that the Fisher information of a system can be \emph{increased}~by restricting attention to a subsystem. 
This is in particular the case if there is some degenerate symmetry in the system, so that the Fisher information is zero for the whole system, but may be positive for a subsystem.  
On the other hand, when the Fisher information of a system is decreased by restricting attention to a single metastable subsystem, we explain in Section \ref{sec:decomp}~how the theory outlined in this paper can still be used to improve the PCR of the inference of a metastable diffusion.

\subsection{Proof of Theorem \ref{thm:quasiposteriorconsistency}}
From Proposition \ref{thm:posteriorconsistency}, it immediately follows that a posterior distribution constructed with respect to $(\hat{\pr}_\theta)_{\theta\in\Theta}$ must exhibit posterior consistency. 
To be precise, letting 
\begin{align}
&\hat{\pr}_\theta^t\,\coloneqq\,\hat{\pr}_\theta\left[X_{[0,t]}\in\,\cdot\,\right], \qquad \hat{L}_t(X \mid \theta)\,\coloneqq\,\frac{d\hat{\pr}_\theta^t}{d\hat{\pr}_{\theta_0}^t}(X),\qquad 
\hat{\L}_t(X \mid \pi_0)\,\coloneqq\,\int_\Theta \hat{L}_t(X \mid \theta)\,\pi_0(d\theta), \nonumber\\
&\hat{\pi}_t(d\theta\mid X)\,\coloneqq\,\frac{1}{\hat{\L}_t(X\mid \pi_0)}\hat{L}_t(X \mid \theta)\,\pi_0(d\theta), 
\label{eq:hatpi}
\end{align}
we have the following corollary to Proposition \ref{thm:posteriorconsistency}. 

\begin{cor}
\label{cor:modifiedconsistency}
Let Assumption \ref{ass:quasi}~hold, and let $\hat{F}_t,\,\hat{U}_t$ be as defined in \eqref{eq:hatF}~for fixed $\alpha\in(0,1)$ and $\delta>0$. 
Let $\varepsilon_{\hat{\Delta}},\,\varepsilon_{\hat{r}},$ be as above. 
Then, $\hat{F}_t\rightarrow\{\theta_0\}$ and $\hat{U}_t\rightarrow\Theta$ as $t\rightarrow\infty$, and 
\begin{equation}
\label{eq:hatconsistency}
\hat{\pr}_{\theta_0}\left[\hat{\pi}_t(\hat{F}_t \mid X)<1-\hat{H}(t)^{1/2}\right]\,\le\,\hat{H}(t)^{1/2}, 
\end{equation}
where $\hat{H}(t)$, defined in \eqref{eq:hatH}, tends to zero as $t\rightarrow\infty$. 
\end{cor} 

\begin{proof}[Proof of Theorem \ref{thm:quasiposteriorconsistency}]
Since $\pr_\theta[\,\cdot\,,t<\tau]=\hat{\pr}_\theta[\,\cdot\,,t<\tau]$ for $\theta\in\Theta$ and $t>0$, 
\begin{equation}
\label{eq:LhatL}
L_t(X|\theta)\,=\,\hat{L}_t(X|\theta)\qquad\text{ for $X\in\{t<\tau\}$}. 
\end{equation}
Indeed, this is a consequence of the following fact: 
For measures $\mu,\,\nu,$ on a measurable space $(E,\mathcal{B}(E))$ such that $\mu\ll\nu$, we have for $A,B\in\mathcal{B}(E)$ that 
\[
\mu(A\cap B)\,=\,\int_{A\cap B}\frac{d\mu(\cdot)}{d\nu(\cdot)}(x)\,d\nu(x), 
\]
and 
\[
\mu(A\cap B)\,=\,\int_A \frac{d\mu(\cdot\cap B)}{d\nu(\cdot\cap B)}\,d\nu(x\cap B)\,=\,\int_{A\cap B}\frac{d\mu(\cdot\cap B)}{d\nu(\cdot\cap B)}\,d\nu(x). 
\]
It follows that, for $\nu$-almost all $x\in E$, we have 
\[
\frac{d\mu(\cdot\cap B)}{d\nu(\cdot\cap B)}(x)\,=\,\frac{d\mu(\cdot)}{d\nu(\cdot)}(x)\,1_B(x). 
\]
Consequently, by \eqref{eq:LhatL}, for $X\in\{t<\tau\}$ we have that $\pi_t(d\theta\mid X)=\hat{\pi}_t(d\theta\mid X)$. 
For $t>0$, $a\in(0,1)$, and measurable $F\subset\Theta$, we then have 
\begin{equation}
\label{eq:decomp}
\begin{aligned}
\pr_{\theta_0}\left[\pi_t(F)<a\right]\,&\le\,\pr_{\theta_0}\left[\pi_t(F)<a,\,t<\tau\right] + \pr_{\theta_0}\left[t>\tau\right] \\
&=\,\hat{\pr}_{\theta_0}\left[\hat{\pi}_t(F)<a,\,t<\tau\right] + \pr_{\theta_0}\left[t>\tau\right] \\
&\le\,\hat{\pr}_{\theta_0}\left[\hat{\pi}_t(F)<a\right] + \pr_{\theta_0}\left[t>\tau\right]. 
\end{aligned}
\end{equation}
Taking \eqref{eq:hatconsistency}~\&~\eqref{eq:decomp}~together with $a=1-\hat{H}^{1/2}(t)$ completes the proof of Theorem \ref{thm:quasiposteriorconsistency}. 
\end{proof}

\section{Application to ergodic diffusions}
\label{sec:ergodicdiffusions}
We here apply the theory developed above to diffusions in $\R^d$, parameterized by $\theta\in\Theta\subset\R^p$, and governed by SDEs of the form 
\begin{equation}
\label{eq:SDE}
dX_t\,=\,V_\theta(X_t)\,dt + \sigma\,dW_t 
\end{equation}
over an ambient probability space $(\Omega,\F,\pr_\theta)$.  
$(W_t)_{t\ge0}$ is a standard $d$-dimensional Wiener process, $\sigma>0$ is fixed, and for each $\theta\in\Theta$ the drift coefficient $V_\theta$ is sufficiently regular. 
When \eqref{eq:SDE}~is ergodic, we find precise bounds on the rates of decay of the quantities $\phi_t$, $\varepsilon_r$ and $\varepsilon_\Delta$, appearing in $(\phi,\Delta,r)$-LAN setup of Theorem \ref{thm:ergodic_diffusion}: $\phi_t$ decays at a rate proportional to the Fisher information; $\varepsilon_\Delta$ decays at a rate proportional the spectral gap of the observed process; 
and $\varepsilon_r$ decays at a rate proportional to a combination of the Fisher information and the spectral gap. 
By Proposition \ref{thm:posteriorconsistency}, this allows us to characterize the PCR and PCCR of \eqref{eq:SDE}. 
We also consider cases where the models in \eqref{eq:SDE}~are 
\begin{enumerate}
\item not ergodic, but, prior to an almost surely finite stopping time, are equivalent in law to models which are ergodic; 
\item ergodic, and, up to an almost surely finite stopping time, are equivalent to models which are also ergodic, but converge to their ergodic distributions at much faster rates; 
\item ergodic, and together have zero Fisher information, but prior to a stopping time are equivalent in law to a family of models with positive Fisher information. 
\end{enumerate}
In these cases, we show that the posterior distribution exhibits \emph{metaconsistency}, as defined in Section \ref{sec:intro}, and are able to precisely characterize the m-PCR and m-PCCR. 

In Section \ref{sec:doublewell}~we consider a concrete example of an ergodic diffusion in a double-well potential. 
Using our results, we obtain precise estimates on how much faster the posterior contracts when the diffusion remains in a neighbourhood of only one of its wells. 
The theory outlined in Section \ref{sec:quasiposterior}~then allows us to conclude a much faster m-PCCR compared to the PCCR of this system. 
We also consider an example of an ergodic diffusion in a multi-well potential which is not identifiable, in that the Fisher information $s_1$ is equal to zero for the whole system, while the Fisher information of the restricted system is positive. 

\subsection{Statement of main results}
We assume that for probability measures $\nu$ on $\R^d$, \eqref{eq:SDE}~with initial distribution $\nu$ has a unique solution, denoted $X_t(\nu)$, up to a possibly finite explosion time $\tau_{exp}$. 
See \cite[Chapter 5]{KS14}~for further details on the solution theory of \eqref{eq:SDE}. 
When convenient, we write $X_t=X_t(\nu)$ for the solution process of \eqref{eq:SDE}. 

\begin{ass}
\label{ass:uniformergodic}
For each $\theta\in\Theta$, $\tau_{exp}=\infty$ almost surely, and \eqref{eq:SDE}~is uniformly ergodic, so there exists a unique probability measure $\mu_\theta$ and constants $k(\theta),\,\gamma(\theta)>0$ such that 
\begin{equation}
\label{eq:TV}
\norm*{\pr_\theta\left[X_t\in\,\cdot\,\right]-\mu_\theta(\,\cdot\,)}_{TV}\,\le\,k(\theta)e^{-\gamma(\theta)t}\qquad\text{ for $t>0$}. 
\end{equation}
For $\theta\in\Theta$ let $\pr_\theta^t=\pr_\theta\left[X_{[0,t]}(\nu)\in\,\cdot\,\right]$ denote the distribution of the solution over the time interval $[0,t]$ with arbitrary initial distribution $\nu$ (which may depend on $\theta\in\Theta$). 
\end{ass}

Slightly abusing terminology, we refer to the constant $\gamma(\theta)$ appearing in \eqref{eq:TV}~as the \emph{spectral gap}~of \eqref{eq:SDE}. 
See Appendix \ref{sec:spectral}~for further discussion of the constants $k(\theta),\,\gamma(\theta)$. 

In addition to Assumption \ref{ass:uniformergodic}, we require a condition guaranteeing the identifiability of the parameters $\Theta$ in a statistical model. 
Below, we assume that $d\ge p$. 
If we were to relax this assumption, our statistical model would be \emph{locally asymptotically mixed normal}, instead of simply LAN \cite{O15}. 
That is, the limiting random variable $\Delta_\infty$ in the definition of LAN models would have a nontrivial covariance matrix. 
As the LAMN setting complicates the exposition of this section, we assume $d\ge p$ to ease the presentation. 

\begin{ass}
\label{ass:fisher}
Fix a (random) prior $\pi_0$ with $\pr_{\theta_0}$-almost surely finite moments. 
Let $d\ge p$ and Assumption \ref{ass:uniformergodic}~hold, along with the following.  
\begin{enumerate}[(i)]
\item For $t>0$, we have $\ExpOp_{\theta_0}\left[\norm*{X_t}\right]\le Kt$ for some $K>0$. 
\item For each $x\in\R^d$, the map $\theta\mapsto V_\theta(x)$ is $C^3$, with $C^3$-norm uniformly bounded in $x\in\R^d$, and $\partial_\theta V_{\theta_0}$ is not constant.  
\item The singular values $s_1\le\ldots\le s_p$ of the matrix $\ExpOp_{\theta_0}\left[\partial_\theta V_{\theta_0}(\xi)\right]$, where $\xi$ denotes a random variable distributed as $\mu_{\theta_0}$ under $\pr_{\theta_0}$, are nonzero. 
Letting $T_1\Sigma T_2$ denote the singular value decomposition of $\ExpOp_{\theta_0}\left[\partial_\theta V_{\theta_0}(\xi)\right]$, define 
\begin{equation}
\label{eq:I}
I_{\theta_0}^{-1/2}\,\coloneqq\,\sigma\, T_2^*\tilde{\Sigma}T_1^*I_{d,p}, 
\end{equation}
where $\tilde{\Sigma}$ is the $p\times d$ matrix with the inverses of the singular values of $\ExpOp_{\theta_0}\left[\partial_\theta V_{\theta_0}(\xi)\right]$ on its diagonal, and $I_{d,p}$ denotes the $d\times p$ matrix with ones on the diagonal, and zeros elsewhere. 
\item $\pi_0(\theta)=\rho_0(u)\overline{\pi}(d\theta)$, where $\overline{\pi}$ is deterministic. 
\end{enumerate}
\end{ass}

Slightly abusing terminology, we refer to $I_{\theta_0}$ in Assumption \ref{ass:fisher}(ii) as the \emph{Fisher information matrix}, and $s_1$ as the \emph{Fisher information}~of the statistical model consisting of the SDEs \eqref{eq:SDE}. 
For further justification, compare the matrix $\ExpOp_{\theta_0}\left[\partial_\theta V_{\theta_0}(\xi)\right]$ with the ``Fisher information matrix'' of \cite{K13}. 

\begin{thm}
\label{thm:ergodic_diffusion}
Let the family of SDEs \eqref{eq:SDE}~parameterized by $\theta\in\Theta$ satisfy Assumption \ref{ass:fisher}. 
Then, the family of posterior distributions $(\pi_t(\,\cdot\mid X))_{t\ge0}$, constructed as in \eqref{eq:pi}, is consistent at each $\theta_0\in\Theta$, so that for fixed $\delta>0$, $\alpha\in(0,1)$, and arbitrary $t\ge0$, letting 
\begin{equation}
\label{eq:FU}
\epsilon_t\,\coloneqq\,\delta/s_1^{\alpha}t^{\alpha/2}, 
\qquad F_t\,\coloneqq\, B_{\epsilon_t}(\theta_0), \quad\text{ and }\quad U_t\,\coloneqq\,t^{1/2}I_{\theta_0}^{1/2}B_{\epsilon_t}(0), 
\end{equation}
it holds that $F_t\rightarrow\{\theta_0\}$ and $U_t\rightarrow\Theta$ as $t\rightarrow\infty$, and 
\[
\pr_{\theta_0}\left[\pi_t(F_t \mid X)<1-H(t)^{1/2}\right]\,\le\,H(t)^{1/2}. 
\]
where 
\[
H(t)\,\coloneqq\,
C\left(
\left(\frac{1-e^{-\gamma(\theta_0)t}}{\gamma(\theta_0)t}\right)^{1/2}
+ \left(\frac{1-e^{-\gamma(\theta_0)t}}{\gamma(\theta_0)t}\right)
+ \frac{1}{s_1^2t} + \frac{1}{s_1^3t^{3/2}} +\ExpOp_{\theta_0}\left[\pi_0(U_t^c)\right]\right), 
\]
for a constant $C>0$. 
\end{thm}

The constant $C$ appearing in Theorem \ref{thm:ergodic_diffusion}~will be discussed further, in the context of a concrete example, in Section \ref{sec:doublewell}~below. 

\begin{rmk}
Note that for many choices of $\pi_0$, $\ExpOp\left[\pi_0(U_t^c)\right]$ and $\ExpOp\left[\pi_0(\hat{U}_t^c)\right]$ decay to zero rapidly; for instance when $\ExpOp\left[\pi_0\right]$ is Gaussian, or when $\ExpOp\left[\pi_0\right]$ is the uniform distribution on a bounded neighbourhood of $\theta_0$, in which latter case $\ExpOp\left[\pi_0(\hat{U}_t^c)\right]$ is zero after finite time. 
Of the other four terms in the definition of $H(t)$, two depend on $\gamma(\theta_0)$, and two depend on $s_1$. 
In metastable systems, it is well known that $\gamma(\theta_0)$ is exponentially small in $\sigma>0$, slowing down the PCCR. 
The following assumption describes a setup that may be used to get around this problem, at least on finite time scales. 
\end{rmk}


\begin{ass}
\label{ass:metafisher}
There exists a domain $O\subset\R^d$ and a family of vector fields $\hat{V}_\theta$, $\theta\in\Theta$, such that $\hat{V}_\theta=V_\theta$ on $O$ and the family of SDEs
\begin{equation}
\label{eq:tildeX}
d\hat{X}\,=\,\hat{V}_\theta(\hat{X})\,dt + \sigma\,dW 
\end{equation}
satisfy Assumption \ref{ass:fisher}, with constant $\hat{\gamma}(\theta)$ and Fisher information $\hat{s}_1$. 
We define the exit time of \eqref{eq:SDE}~from $O$ as $\displaystyle{\tau\,\coloneqq\,\inf\left\{t>0\,:\,X_t\in\partial O\right\}}$. 
\end{ass}

This assumption allows us to state our main result on metastable diffusions as follows. 

\begin{thm}
\label{thm:quasiergodic_diffusion}
Let Assumption \ref{ass:metafisher}~hold. 
Then, for fixed $\delta>0$ and $\alpha\in(0,1)$, letting 
\[
\hat{\epsilon}_t\,\coloneqq\,\delta/\hat{s}_1^{\alpha}t^{\alpha/2},\qquad\hat{F}_t\,\coloneqq\, B_{\hat{\epsilon}_t}(\theta_0), \quad\text{ and }\quad \hat{U}_t\,\coloneqq\,t^{1/2}\hat{I}_{\theta_0}^{1/2}B_{\hat{\epsilon}_t}(0), 
\]
it holds that $\hat{F}_t\rightarrow\{\theta_0\}$ and $\hat{U}_t\rightarrow\Theta$ as $t\rightarrow\infty$, and 
\begin{equation}
\label{eq:hatPP}
\pr_{\theta_0}\left[\pi_t(\hat{F}_t)\,<\,1-\hat{H}(t)^{1/2}\right]\,\le\,\hat{H}(t)^{1/2} + \pr[t>\tau], 
\end{equation}
where here, for a constant $\hat{C}>0$, 
\[
\hat{H}(t)\,\coloneqq\,
\hat{C}\left(\left(\frac{1-e^{-\hat{\gamma}(\theta_0)t}}{\hat{\gamma}(\theta_0)t}\right)^{1/2} + \left(\frac{1-e^{-\hat{\gamma}(\theta_0)t}}{\hat{\gamma}(\theta_0)t}\right) + \frac{1}{\hat{s}_1^2t} + \frac{1}{\hat{s}_1^3t^{3/2}} + \ExpOp_{\theta_0}\left[\pi_0(\hat{U}_t^c)
\right]\right)
\]
\end{thm}

The constant $\hat{C}$ appearing in Theorem \ref{thm:quasiergodic_diffusion}~will be compared with the constant $C$ from Theorem \ref{thm:ergodic_diffusion}, in the context of a concrete example, in Section \ref{sec:doublewell}~below. 

Through Theorem \ref{thm:ergodic_diffusion}, we see that in the case of diffusions, the learning rate is determined by 1.~the spectral gap of our observed process and 2.~the smallest singular value of the \fim. 
In the case of diffusions with no true stable state, the spectral gap may not exist, while in the case of a stable diffusions with several metastable states contained within a globally stable domain, the spectral gap may be quite small. 
Additionally, as we see in the following example, metastability may lead to problems of parameter identifiability when each of the metastable regions of a systems have qualitatively similar behaviour. 
However, if one restricts attention to a single metastable region of the system, both of these issues may be overcome, as the restricted system will be truly stable with a relatively large spectral gap. 
As we illustrate in the examples of Section \ref{sec:doublewell}, this may substantially improve the learning rate in certain scenarios. 
The case where restricting attention to a single metastable state increases the spectral gap of the system, but decreases the Fisher information, is discussed in Section \ref{sec:decomp}.

\subsection{Proofs of Theorems \ref{thm:ergodic_diffusion}~\&~\ref{thm:quasiergodic_diffusion}}
To prove Theorem \ref{thm:ergodic_diffusion}, we begin by demonstrating that the family of SDEs given by \eqref{eq:SDE}~is LAN under Assumption \ref{ass:fisher}. 
By Girsanov's Theorem, the likelihood ratio of a given model probability distribution $\pr_\theta^t$ given $\pr_{\theta_0}^t$ is the Radon-Nikodym derivative 
\[
\frac{d\pr^t_\theta}{d\pr_{\theta_0}^t}\,=\,\exp\left(-\frac{1}{2\sigma^2}\int_0^tV_\theta^2-V_{\theta_0}^2\,ds + \frac{1}{\sigma^2}\int_0^t V_\theta-V_{\theta_0}\,dX_s\right), 
\]
where we use the shorthand $V_\theta=V_\theta(X_s)$, $V_\theta^2=\norm*{V_\theta}^2$, and $V_\theta\,dX_s=\inner*{V_\theta}{dX_s}$. 
See \cite[Theorem 1.12]{K13}~for a proof in the setting of one-dimensional diffusions. 
Under Assumption \ref{ass:fisher}, an asymptotic version of the following result on the LAN property of $(\pr_\theta)_{\theta\in\Theta}$ is known. 
However, in this section we sharpen the result, and obtain precise bounds that hold in the finite time regime. 

\begin{prop}[Local asymptotic normality]
\label{thm:ergodicLAN}
Under Assumption \ref{ass:fisher}, $(\pr_\theta)_{\theta\in\Theta}$ is LAN at each $\theta_0\in\Theta$. 
In particular, define $\phi_t=\phi_t(\theta_0)\,\coloneqq\,t^{-1/2}I_{\theta_0}^{-1/2}$, 
let 
\begin{equation}
\label{eq:definitions}
\begin{aligned}
&\Delta_t(X)\cdot u\,\coloneqq\,\int_0^t\frac{\partial_\theta V_{\theta_0}(X_s)[\phi_tu]}{\sigma^2}\left[dX_s-V_{\theta_0}(X_s)\,ds\right],\quad\text{ and }\\
&r_t(X,u)\,\coloneqq\,\left(\norm*{u}^2- \int_0^t\frac{\norm*{\partial_\theta V_{\theta_0}(X_s)[\phi_tu]}^2}{\sigma^2}\,ds\right)\\
&\qquad\qquad\qquad\qquad +\int_0^t\frac{\partial_\theta^2V_{\theta_0}(X_s)[\phi_tu]^2}{\sigma^2}\,[dX_s-V_{\theta_0}(X_s)ds] \\
&\qquad\qquad\qquad\qquad\qquad + O\left(\sigma^3\norm*{\phi_tu}^3(t+X_t)\right)\,. 
\end{aligned}
\end{equation}
Writing $\Delta_t=\Delta_t(X)$ and $r_t=r_t(X,u)$, for all $t>0$, $u\in\R^p$, and $X\in C([0,\infty);\R^d)$ it holds that 
\begin{equation}
\frac{d\pr^t_{\theta_0+\phi_tu}}{d\pr_{\theta_0}^t}\,=\,\exp\left(\Delta_t\cdot u-\norm*{u}^2+r_t\right), 
\end{equation}
where $r_t\rightarrow0$ in $\pi_0\otimes\pr_{\theta_0}$-mean and $\Delta_t$ tends to a $\mathcal{N}(0,I)$ random variable in $\pr_{\theta_0}$-mean square as $t\rightarrow\infty$. 
\end{prop}
\begin{proof}
We again use the shorthand $V_\theta=V_\theta(X_t),\,\partial_\theta V_\theta=\partial_\theta V_\theta(X_t)$. 
Taking second order Taylor approximations of $V_\theta$ and $V_\theta^2$ in the parameter $\theta$ centered at $\theta_0$, we obtain 
\[
\begin{aligned}
\frac{d\pr_\theta^t}{d\pr_{\theta_0}^t}\,&=\,\exp\left(-\frac{1}{2\sigma^2}\int_0^t V_\theta^2-V_{\theta_0}^2\,ds + \frac{1}{\sigma^2}\int_0^tV_\theta-V_{\theta_0}\,dX_s\right)\\ 
&=\,\exp
\bigg(
-\frac{1}{\sigma^2}\int_0^tV_{\theta_0}\partial_\theta V_{\theta_0}[\theta-\theta_0] + \left[\norm*{\partial_\theta V_{\theta_0}}^2+V_{\theta_0}\partial_\theta^2V_{\theta_0}\right]\left[\theta-\theta_0\right]^2 + O(\norm*{\theta-\theta_0}^3)\,ds \\
&\qquad\qquad+\frac{1}{\sigma^2}\int_0^t\partial_\theta V_{\theta_0}[\theta-\theta_0] + \partial_\theta^2V_{\theta_0}[\theta-\theta_0]^2+O(\norm*{\theta-\theta_0}^3)\,dX_s 
\bigg)\\
&=\,
\exp
\bigg(
-\frac{1}{\sigma^2}\int_0^tV_{\theta_0}\partial_\theta V_{\theta_0}[\theta-\theta_0] + \left[\norm*{\partial_\theta V_{\theta_0}}^2+V_{\theta_0}\partial_\theta^2V_{\theta_0}\right]\left[\theta-\theta_0\right]^2\,ds  \\
&\qquad\qquad+\frac{1}{\sigma^2}\int_0^t\partial_\theta V_{\theta_0}[\theta-\theta_0] + \partial_\theta^2V_{\theta_0}[\theta-\theta_0]^2\,dX_s 
+ O\left(\sigma^{-2}\norm*{\theta-\theta_0}^3(t+X_t)\right) 
\bigg)\\
&=\,\exp
\bigg(
\int_0^t\frac{\partial_\theta V_{\theta_0}}{\sigma^2}[\theta-\theta_0]\,[dX_s-V_{\theta_0}ds] - \int_0^t\frac{\norm*{\partial_\theta V_{\theta_0}[\theta-\theta_0]}^2}{\sigma^2}\,ds \\
&\qquad\qquad +\int_0^t\frac{\partial_\theta^2V_{\theta_0}[\theta-\theta_0]^2}{\sigma^2}\,[dX_s-V_{\theta_0}ds] 
+ O\left(\sigma^{-2}\norm*{\theta-\theta_0}^3(t+X_t)\right)
\bigg). 
\end{aligned}
\]
Now, making the transformation $u\coloneqq\phi_t^{-1}[\theta-\theta_0]$ (noting that $I_{\theta_0}^{-1/2}$ is a $p\times p$ matrix with non-zero singular values, and hence invertible), we obtain 
\[
\begin{aligned}
\frac{d\pr^t_{\theta_0+\phi_tu}}{d\pr_{\theta_0}^t}\,&=\,\exp
\bigg(
\int_0^t\frac{\partial_\theta V_{\theta_0}}{\sigma^2}[\phi_tu]\,[dX_s-V_{\theta_0}ds] - \int_0^t\frac{\norm*{\partial_\theta V_{\theta_0}[\phi_tu]}^2}{\sigma^2}\,ds \\
&\qquad\qquad +\int_0^t\frac{\partial_\theta^2V_{\theta_0}[\phi_tu]^2}{\sigma^2}\,[dX_s-V_{\theta_0}ds] 
+ O\left(\sigma^{-2}\norm*{\phi_tu}^3(t+X_t)\right) 
\bigg)\\
&=\,\exp\left(\Delta_t\cdot u - \norm*{u}^2 + r_t\right), 
\end{aligned}
\]
by the definitions of $\Delta_t$ and $r_t$ in \eqref{eq:definitions}. 
It remains to be shown that $\Delta_t$ converges to a standard normal random variable in $\pr_{\theta_0}$-mean square, and that $r_t$ converges to zero in $\pi_0\otimes\pr_{\theta_0}$-mean. 
Let $\Delta_\infty$ and $\Delta_\infty'$ be standard $p$- and $d$-dimensional standard normal random variables, respectively. 
Note that for a standard $d$-dimensional Brownian motion $(W_t)_{t\ge0}$, the equalities 
\[
dX_s-V_{\theta_0}\,ds\,=\,\sigma\,dW_s\qquad\text{ and }\qquad \Delta_\infty\cdot u\,=\,u\cdot I_{p,d}\cdot\Delta_\infty'\,=\,t^{-1/2}\int_0^t u\cdot I_{p,d}\cdot dW_s 
\]
hold in $\pr_{\theta_0}$-law. 
Consequently, from the definition of $\phi_t$ we have by It{\^o}'s isometry that 
\begin{equation}
\label{eq:Deltadiff0}
\begin{aligned}
\ExpOp_{\theta_0}\left[\norm*{(\Delta_t-\Delta_\infty)\cdot u}^2\right]\,&=\,\ExpOp_{\theta_0}\left[\norm*{t^{-1/2}\int_0^t\sigma^{-1}\partial_\theta V_{\theta_0} \cdot I_{\theta_0}^{-1/2}\cdot u\cdot\,dW_s - t^{-1/2}\int_0^t u\cdot I_{p,d}\cdot\,dW_s}^2\right]\\
&=\,\ExpOp_{\theta_0}\left[t^{-1}\int_0^t\norm*{\sigma^{-1}\partial_\theta V_{\theta_0}\cdot I_{\theta_0}^{-1/2}\cdot u - u\cdot I_{p,d}}^2\,ds\right] 
\end{aligned}
\end{equation} 
for $u\in\Theta$. 
Taking $t\rightarrow\infty$ and applying Birkhoff's ergodic theorem, together with Fatou's lemma, 
\begin{equation}
\label{eq:Deltadiff}
\ExpOp_{\theta_0}\left[\norm*{(\Delta_t-\Delta_\infty)\cdot u}^2\right]\,\xrightarrow[t\rightarrow\infty]{}\,
\norm*{\ExpOp_{\theta_0}\left[\sigma^{-1}\partial_\theta V_{\theta_0}(\xi_1)\right]\cdot I_{\theta_0}^{-1/2}\cdot u-u\cdot I_{p,d}}^2, 
\end{equation}
where $\xi_1$ is another random variable distributed as $\mu_{\theta_0}$ under $\pr_{\theta_0}$. 
The last quantity on the right hand side of \eqref{eq:Deltadiff}~is zero, by the definition of $I_{\theta_0}^{-1/2}$ and $I_{p,d}$. 
We defer the proof of the decay of $r_t$ to Proposition \ref{thm:ergodicdecay}~below. 
\end{proof}

We now study the rates of decay of $\varepsilon_\Delta$ and $\varepsilon_r$ under Assumption \ref{ass:fisher}. 
This result demonstrates how the certainty of the contraction of the posterior distribution about the true parameter value is controlled by the spectral gap $\gamma(\theta_0)$ of \eqref{eq:SDE}~with $\theta=\theta_0$ and the smallest singular value of $\ExpOp_{\theta_0}\left[\partial_\theta V_{\theta_0}(\xi)\right]$. 

\begin{prop}
\label{thm:ergodicdecay}
Let Assumption \ref{ass:fisher}~hold. 
Then, with $\gamma(\theta)$ as the constant appearing in \eqref{eq:TV}, there exist $c_\Delta,c_r>0$ depending solely on $V_{\theta_0}$ and $\pi_0$, such that 
\begin{equation}
\label{eq:ergodicepsilons}
\begin{aligned}
\varepsilon_\Delta(t)\,&\coloneqq\,\ExpOp_{\theta_0}\left[\norm*{\left(\Delta_t-\Delta_\infty\right)}^2\right]\,\le\,c_\Delta\frac{1-e^{-\gamma(\theta_0)t}}{\gamma(\theta_0)t} 
\qquad\text{and}\\
\varepsilon_r(t)\,&\coloneqq\,\ExpOp_{\theta_0}\left[\int_\Theta \abs*{r_t}\,\pi_0(du)\right]\,\le\,c_r\left(\frac{1-e^{-\gamma(\theta_0)t}}{\gamma(\theta_0)t}+\frac{1}{s_1^2t} + \frac{1}{s_1^3t^{3/2}}\right). 
\end{aligned}
\end{equation}
\begin{proof}
As in \eqref{eq:Deltadiff0}, using $u\cdot I_{p,d}=I_{d,p}\cdot u$, we have that 
\[
\ExpOp_{\theta_0}\left[\norm*{(\Delta_t-\Delta_\infty)\cdot u}^2\right]\,=\,\ExpOp_{\theta_0}\left[t^{-1}\int_0^t\norm*{\left(\partial_\theta V_{\theta_0}\cdot\sigma^{-1}I_{\theta_0}^{-1/2} - I_{d,p}\right)\cdot u}^2\,ds\right]. 
\]
Consequently, 
\[
\begin{aligned}
\ExpOp_{\theta_0}\left[\norm*{\Delta_t-\Delta_\infty}^2\right]\,&=\,\ExpOp_{\theta_0}\left[t^{-1}\int_0^t\norm*{\partial_\theta V_{\theta_0}\cdot\sigma^{-1}I_{\theta_0}^{-1/2} - I_{d,p}}^2\,ds\right]\\
&\eqqcolon\,\ExpOp_{\theta_0}\left[t^{-1}\int_0^t g_1(X_s)\,ds\right]. 
\end{aligned}
\]
Now, recall the dual representation of the total variation distance: for any two probability measures $P,\,Q,$ on a metric space $M$,  
\[
\norm*{P-Q}_{TV}\,=\,\frac{1}{2F}\sup_{\norm*{f}_\infty\le F}\abs*{P(f)-Q(f)}. 
\]
We assume that $\norm*{g_1}_\infty>0$, otherwise the result is trivial. 
Then, due to \eqref{ass:uniformergodic}~and the fact that $\mu_{\theta_0}(g_1)=0$, which follows from \eqref{eq:Deltadiff}, we see that 
\begin{equation}
\label{eq:spectralbound}
\begin{aligned}
\ExpOp_{\theta_0}\left[t^{-1}\int_0^tg_1(X_s)\,ds\right]\,&=\,t^{-1}\int_0^t\ExpOp_{\theta_0}\left[g_1(X_s)\right]\,ds\\
&\le\,2^{-1}\norm*{g_1}_\infty^{-1}\int_0^1k(\theta)e^{-\gamma(\theta_0)ts}\,ds\\
&=\,2^{-1}\norm*{g_1}_\infty^{-1}k(\theta)\frac{1}{\gamma(\theta_0)t}\left(1-e^{-\gamma(\theta_0)t}\right), 
\end{aligned}
\end{equation}
completing the proof for $\varepsilon_\Delta(t)$. 
Indeed, we may take $c_\Delta\,=\,2^{-1}\norm*{g_1}^{-1}_\infty k(\theta)$, which is well defined when $\partial_\theta V_{\theta_0}$ is not a constant. 

The proof for the first term in $\varepsilon_r(t)$ follows the same lines. 
Indeed, using Assumption \ref{ass:fisher}, observe that the first term is 
\[
\begin{aligned}
&\ExpOp_{\theta_0}\left[\int_\Theta\abs*{\norm*{u}^2-\int_0^t\sigma^{-2}\norm*{\partial_\theta V_{\theta_0}(X_s)[\phi_tu]}^2\,ds}\,\pi_0(du)\right]\\
&\qquad\le\,\int_{\Theta}\ExpOp_{\theta_0}\left[
\abs*{\norm*{u}^2-\int_0^t\sigma^{-2}\norm*{\partial_\theta V_{\theta_0}(X_s)[\phi_tu]}^2\,ds}\rho_0(u)
\right]\,\overline{\pi}(du)\\
&\qquad\le\,\int_{\Theta}\ExpOp_{\theta_0}\left[
t^{-1}\int_0^t\abs*{\norm*{u}^2-\norm*{\partial_\theta V_{\theta_0}(X_s)[\sigma^{-1}I_{\theta_0}^{-1/2}u]}^2}\rho_0(u)\,ds
\right]\,\overline{\pi}(du)\\
&\qquad\eqqcolon\,\int_\Theta\ExpOp_{\theta_0}\left[t^{-1}\int_0^tg_2(u,X_s)\,ds\right]\,\overline{\pi}(du)\\
&\qquad\le\,\left(2^{-1}\int_\Theta\ExpOp_{\theta_0}\left[\norm*{g_2(u,\,\cdot\,)}_\infty^{-1}\right]\,\overline{\pi}(du)\right)k(\theta)\frac{1}{\gamma(\theta_0)t}\left(1-e^{-\gamma(\theta_0)t}\right)\\
&\qquad\eqqcolon\,C'\frac{1}{\gamma(\theta_0)t}\left(1-e^{-\gamma(\theta_0)t}\right), 
\end{aligned}
\]
where, similar to the constant $c_\Delta$, note that $C'$ is well defined when $\partial_\theta V_{\theta_0}$ is not a constant. 

To treat the $\pi_0\otimes\pr_{\theta_0}$-expectation of the second term, we again use that $dX_s-V_{\theta_0}\,ds=\sigma\,dW_s$ in $\pr_{\theta_0}$-law, and use H{\"o}lder's inequality and It{\^o}'s isometry to compute  
\[
\begin{aligned}
&\ExpOp_{\theta_0}\left[\int_\Theta\,\norm*{\int_0^t\sigma^{-2}\partial_\theta^2V_{\theta_0}(X_s)[\phi_tu]^2\,dW_s}\pi_0(d\theta)\right]\\
&\qquad\le\,\int_\Theta\ExpOp_{\theta_0}\left[\norm*{\int_0^t \partial_\theta^2V_{\theta_0}(X_s)[\sigma^{-1}\phi_tu]^2\rho_0(u)^{1/2} \,dW_s}^2\right]\,\overline{\pi}(du)^{1/2}\\
&\qquad\le\int_\Theta\ExpOp_{\theta_0}\left[t^{-2}\int_0^t\norm*{\partial_\theta^2V_{\theta_0}(X_s)[\sigma^{-1}I_{\theta_0}^{-1/2}u]^2\rho_0(u)^{1/2}}^2\,ds\right]\,\overline{\pi}(du)^{1/2}\\
&\qquad\le\norm*{V_{\theta_0}}_{C^2}\sigma^{-2}\norm*{I_{\theta_0}^{-1/2}}^2\ExpOp_{\theta_0}\left[\int_\Theta \norm*{u}^4\,\pi_0(du)\right]^{1/2}t^{-1}\\
&\qquad\eqqcolon\,C''s_1^{-2}t^{-1}. 
\end{aligned}
\]

Finally, Assumption \ref{ass:fisher}~allows us to conclude that there exists a constant $c>0$ such that the $\pi_0\otimes\pr_{\theta_0}$-expectation of the third term in the definition of $r_t$ in \eqref{eq:definitions}~is bounded by 
\[
\begin{aligned}
\ExpOp_{\theta_0}\left[\int_\Theta \norm*{V_{\theta_0}}_{C^3}(t+X_t)\sigma^{-3}\norm*{\phi_tu}^3\,\pi_0(d\theta)\right]\,&\le\,\left(\norm*{V_{\theta_0}}_{C^3}\pi_0\left(\norm*{u}^3\right)K\right)s_1^{-3}t^{-3/2}\\
&\eqqcolon\,C'''s_1^{-3}t^{-3/2}. 
\end{aligned}
\]
Taking $c_r\coloneqq\,\max\{C',C'',C'''\}$ completes the proof. 
\end{proof}
\end{prop}

\begin{rmk}
In the above, we see that if one assumes that $\mu_\theta$ is the initial distribution of the process -- that is, that the process has relaxed to its ergodic state -- the terms involving $\gamma(\theta_0)$ in \eqref{eq:ergodicepsilons}~vanish, as one should expect. 
Hence, when the spectral gap $\gamma(\theta_0)$ is large, considering the dynamics prior to relaxation to a stationary state has a substantial effect on the PCCR. 
\end{rmk}

We can now prove the main theorems of this section. 
With the work done above, Proposition \ref{thm:posteriorconsistency}~\&~\ref{thm:quasiposteriorconsistency}~follow with little effort. 

\begin{proof}[Proof of Theorem \ref{thm:ergodic_diffusion}]
From the LAN property of Proposition \ref{thm:ergodicLAN}~and the bounds on the rates of decay of $\varepsilon_r,\,\varepsilon_\Delta$, in Proposition \ref{thm:ergodicdecay}, the result follows as a corollary to Proposition \ref{thm:posteriorconsistency}. 
Note in particular that, in the setting of this section, $\abs*{\phi_t}=s_1^{-1}t^{-1/2}$, so $F_t$ and $U_t$ in \eqref{eq:FU}~agree with the quantities denoted by the same symbols in \eqref{eq:Ft}. 
\end{proof}

\begin{proof}[Proof of Theorem \ref{thm:quasiergodic_diffusion}]
Note that Theorem \ref{thm:ergodic_diffusion}~applies to \eqref{eq:tildeX}. 
Consequently, by Assumption \ref{ass:metafisher}, we may apply Theorem \ref{thm:quasiposteriorconsistency}. 
\end{proof}

\subsection{Learning diffusions in a non-degenerate double well}
\label{sec:doublewell}
Let us now consider a concrete example to which the theory developed above applies. 
Theorems \ref{thm:ergodic_diffusion}~\&~\ref{thm:quasiergodic_diffusion}, together with the small noise asymptotics of the spectra of diffusions generators summarized in Appendix \ref{sec:spectral}, allow us to precisely estimate PCRs, PCCRs, m-PCRs, and m-PCCRs. 
In this example we see that the PCR and m-PCR are similar, while the m-PCCR decays at a rate several orders of magnitude faster than that of the PCCR. 

Consider a diffusion on $\R$, governed by the SDE 
\begin{equation}
\label{eq:doublewell}
dX_t\,=\,\nabla U_\theta(X_t)\,dt + \sigma\,dW_t, 
\end{equation}
where $\sigma>0$ is a small parameter controlling noise amplitude, which we treat as fixed, $\theta\in\Theta=\R$, and $U_\theta$ is the double well potential 
\[
U_\theta(x)\,=\,\left(x^2+2x+\theta\right)\left(x^2-2x\right). 
\]
With parameter $\theta_0=0$, the system \eqref{eq:doublewell}~exhibits two metastable states at $\pm\sqrt{2}$, at the bottoms of two wells with depths equal to $4$. 

We assume that we have obtained a time series generated by \eqref{eq:doublewell}~with $\sigma=0.1$ and $\theta_0=0$, and wish to learn the parameter $\theta\in\Theta$ from an initial guess given in the form of a Gaussian prior $\pi_0$. 
For any $\theta\in\Theta$, the system \eqref{eq:doublewell}~is ergodic, with a unique ergodic measure given by 
\[
\mu_{\theta}(dx)\,=\,e^{-\sigma^{-2}U_{\theta}(x)}\,dx. 
\]
It is straightforward to check that \eqref{eq:doublewell}~satisfies the conditions of Theorem \ref{thm:uniformergodic}~in Appendix \ref{sec:spectral}, so that \eqref{eq:doublewell}~satisfies Assumption \ref{ass:uniformergodic}~for some $c(\theta),\,\gamma(\theta)>0$. 
Indeed, up to a small multiplicative constant which is less than $2$, quantified in Appendix \ref{sec:spectral}, we have  
\begin{equation}
\label{eq:gamma}
\gamma(\theta_0)\,\simeq\,\left(8^{3/2}\pi^{-1}\right)e^{-8\sigma^{-2}}\qquad\text{ for the parameter value $\theta_0=0$}. 
\end{equation}
In particular, $\gamma(\theta_0)$ becomes arbitrarily small for sufficiently small $\sigma>0$, due to the bistable nature of our system. 
See \eqref{eq:gammaL}~in Appendix \ref{sec:spectral}. 
For $\sigma=0.1$, this is 
\[
\gamma(\theta_0)\,\simeq\,(1.389\ldots)\times 10^{-11}. 
\]
The Fisher information may be computed, in this case, as 
\[
s\,\coloneqq\,\abs*{\ExpOp_{\theta_0}\left[\partial_\theta\nabla U_\theta(\xi)\right]}\,=\,\abs*{2\left(1-\ExpOp_{\theta_0}\left[\xi\right]\right)}\,=\,2, 
\]
where $\xi$ is an $\R$-valued random variable distributed as $\mu_{\theta_0}$ under $\pr_{\theta_0}$, so $\ExpOp_{\theta_0}[\xi]=0$, by the symmetry of $\mu_{\theta_0}$. 
As in Proposition \ref{thm:posteriorconsistency}~for $\alpha=1/2$ and some $\delta>0$, we define $\phi_t=t^{-1/2}s^{-1}$ and 
\[
F_t\,\coloneqq\,t^{-1/4}s^{-1/2}B_\delta(\theta_0),\qquad\qquad U_t\,\coloneqq\,t^{1/4}s^{1/2} B_\delta(\theta_0). 
\]
Letting $\pi_0$ be a standard normal Gaussian on $\Theta=\R$, we have 
\[
\pi_0(U_t)\,=\,\frac{2}{\sqrt{2\pi}}\int_0^{t^{1/4}s^{1/2}\delta} e^{-\theta^2/2}\,d\theta, 
\]
which can be evaluated numerically. 
To summarize, following Theorem \ref{thm:ergodic_diffusion}, the bound on the contraction certainty of \eqref{eq:doublewell}~is approximately given by 
\begin{equation}
\label{eq:PCCR}
\begin{aligned}
&H(t)\,=\,C\left(\left(\frac{1-e^{-\gamma(\theta_0)t}}{\gamma(\theta_0)t}\right)^{1/2} + \frac{1-e^{-\gamma(\theta_0)t}}{\gamma(\theta_0)t} + \frac{1}{s_1^2t} + \frac{1}{s_1^3t^{3/2}} +\ExpOp_{\theta_0}\left[\pi_0(U_t^c)\right]\right), \\
&\text{ with }\qquad\gamma(\theta_0)\,=\,\left(1.389\ldots\right)\times 10^{-11}\quad\text{ and }\quad s_1\,=\,2,  
\end{aligned}
\end{equation}
where $C>0$ is a constant, discussed in Remark \ref{rmk:ChatC}. 
As seen in Figure \ref{fig:contractionbounds}, this bound remains bad up to relatively large times.  
In Figure \ref{fig:contractionbounds}, we plot the radius of $F_t$ over time as the solid blue line, and the bound \eqref{eq:PCCR}~on the PCCR as the dashed blue line, assuming that $\ExpOp_{\theta_0}\left[\pi_0\right]$ is a standard Gaussian. 
Due to the smallness of $\gamma(\theta)$, the bound on the PCCR is bad over very long time scales. 

\begin{figure}[h!]
\centering
\begin{subfigure}{0.5\textwidth}
\centering
\includegraphics[scale=0.2]{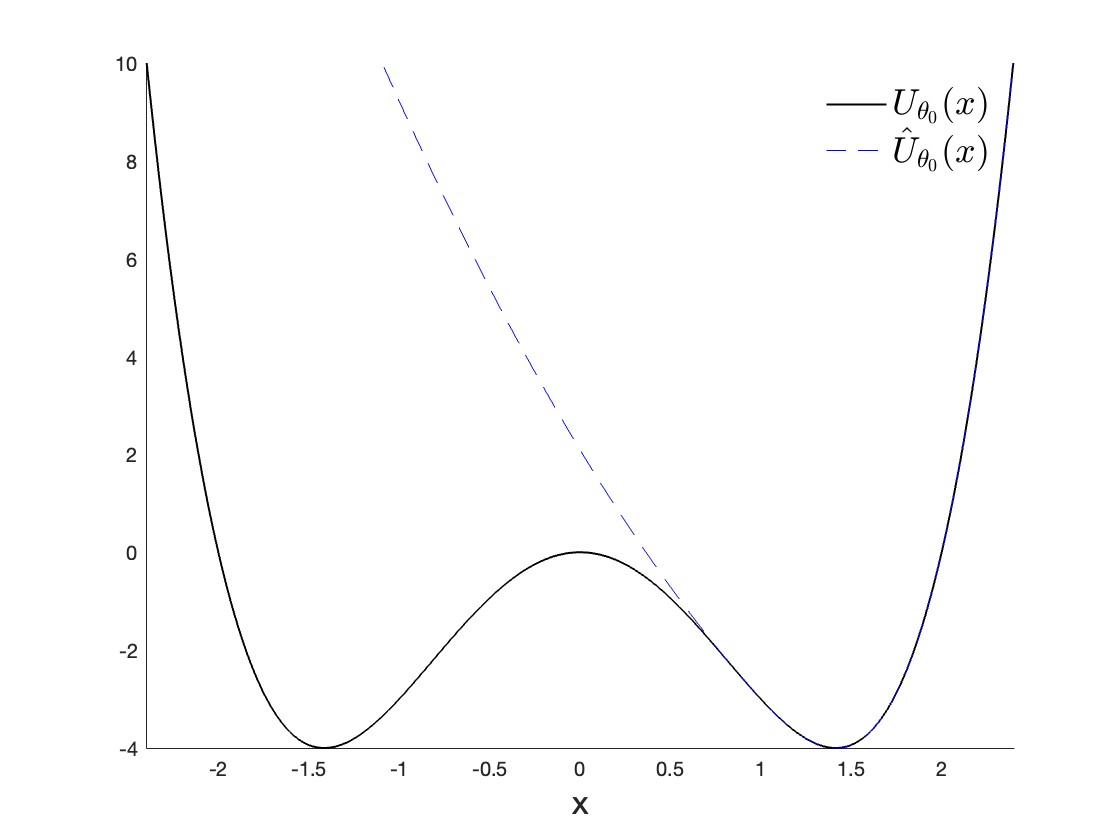}
\caption{}
\label{fig:potentials}
\end{subfigure}%
~ 
\begin{subfigure}{0.5\textwidth}
\centering
\includegraphics[scale=0.2]{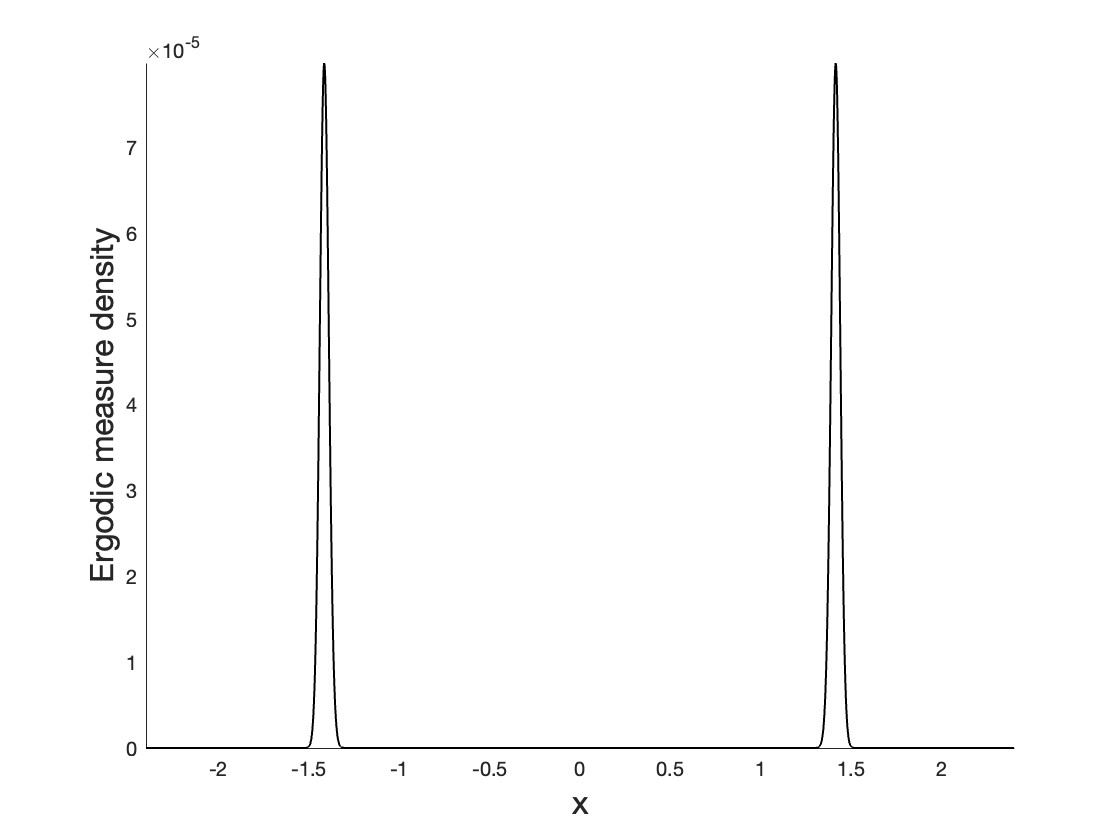}
\caption{}
\label{fig:density}
\end{subfigure}
\caption{
(a) Solid line: the double well potential $U_\theta$ with the true parameter value $\theta_0=0$. 
Dashed line: the modified single well potential $\hat{U}_\theta$ with true parameter $\theta_0=0$. 
(b) The Lebesgue density of the ergodic measure $\mu_0(dx)=e^{-U_{\theta_0}(x)/\beta}\,dx$. 
}
\label{fig:doublewell}
\end{figure}

Now, we let $\hat{U}_\theta$ be a smooth potential such that $\hat{U}_\theta(x)= U_\theta(x)$ for $x\in B_{1/2}(\sqrt{2})$, and $\hat{U}_\theta(x)$ is uniformly convex on $\R$. 
For instance, take the modified potential 
\begin{equation}
\label{eq:C2cutoff}
\hat{U}_\theta(x)\,\coloneqq\,
\begin{cases}
U_\theta(x)&\text{ for }\,x\ge\sqrt{2}-1/2, \\
U_\theta(\sqrt{2}+1/2) + U_\theta'(\sqrt{2}+1/2)(x-\sqrt{2}+1/2) &\\
\qquad\qquad+ U_\theta''(\sqrt{2}-1/2)(x-\sqrt{2}+1/2)^2&\text{ for }\,x<\sqrt{2}-1/2,  
\end{cases}
\end{equation}
shown as the dashed line in Figure \ref{fig:potentials}. 
For $\theta\in\Theta$, consider the modified systems 
\begin{equation}
\label{eq:cutoffwell}
dX_t\,=\,\nabla\hat{U}_\theta(X_t)\,dt + \sigma\,dW_t. 
\end{equation}
Let $\pr_\theta^t$ denote the distribution of the process governed by \eqref{eq:doublewell}~over the time interval $[0,t]$, and let $\hat{\pr}_{\theta}^t$ denote the distribution of the process governed by \eqref{eq:cutoffwell}~over the time interval $[0,t]$. 
Letting $\tau$ be the exit time of $(X_t)_{t\ge0}$ from $B_{1/2}(\sqrt{2})$, we see that $\pr_\theta^t$ and $\hat{\pr}_{\theta}^t$ satisfy Assumption \ref{ass:quasi}. 

By Theorem \ref{thm:escaperate}~of Appendix \ref{sec:spectral}, for small $\sigma>0$ we have $\pr_\theta\left[t<\tau\right]\simeq e^{-\lambda(\theta)t}$ up to a multiplicative error term less than $2$, 
where the constant $\lambda(\theta)>0$ is such that at $\theta_0=0$ the approximation 
\[
\lambda(\theta_0)\,\simeq\,\left(8^{3/2}\pi^{-1}\right)e^{-8\sigma^{-2}} 
\]
holds. 
For $\sigma=0.1$, this is 
\[
\lambda(\theta_0)\,\simeq\,(1.389\ldots)\times 10^{-11}.
\]
By Theorem \ref{thm:modergodic}, the constants $\hat{c}(\theta)$ and $\hat{\gamma}(\theta)$ for \eqref{eq:cutoffwell}~in Assumption \ref{ass:uniformergodic}~may be taken such that $\hat{c}(\theta)$ of the same order of magnitude as $c(\theta)$ with an appropriate choice of initial distribution, and 
\[
\hat{\gamma}(\theta_0)\,=\,\inf_{x\in\R^d}\hat{U}_\theta''(x)\,=\,9.991\ldots. 
\]
As $\hat{\gamma}(\theta_0)$ remains bounded well away from zero uniformly in $\sigma>0$, we have a relatively fast PCCR for the posterior distribution associated with \eqref{eq:cutoffwell}. 
Due to the smallness of $\lambda(\theta_0)$, the bound on the PCCR remains small on a timescale of order $10^{11}$. 
Using the cutoff potential \eqref{eq:C2cutoff}~and letting $\hat{\xi}$ be a random variable distributed as $\hat{\mu}_{\theta_0}$ under $\hat{\pr}_{\theta_0}$, we compute the Fisher information numerically as 
\[
\hat{s}\,\coloneqq\,\abs*{\hat{\ExpOp}_{\theta_0}\left[\partial_\theta\hat{U}_{\theta_0}'(\hat{\xi})\right]}\,=\,0.827\ldots. 
\]
To summarize, by Theorem \ref{thm:quasiergodic_diffusion}, the bound on the m-PCCR is therefore approximated by 
\begin{equation}
\label{eq:mPCCR}
\begin{aligned}
&\hat{H}(t)\,=\,
\hat{C}\left(\left(\frac{1-e^{-\hat\gamma(\theta_0)t}}{\hat{\gamma}(\theta_0)t}\right)^{1/2} + \frac{1-e^{-\hat\gamma(\theta_0)t}}{\hat{\gamma}(\theta_0)t} + \frac{1}{\hat{s}_1^2t} + \frac{1}{\hat{s}_1^3t^{3/2}} + \ExpOp_{\theta_0}\left[\pi_0(\hat{U}_t^c)
\right]\right)\\
&\text{ with }\qquad\hat{\gamma}(\theta_0)\,=\,9.991\ldots\quad\text{ and }\quad \hat{s}_1\,=\,0.827\ldots , 
\end{aligned}
\end{equation}
where $\hat{C}>0$ is a constant, which we discuss presently in Remark \ref{rmk:ChatC}. 
In Figure \ref{fig:contractionbounds}, this m-PCCR is plotted as a dashed red line. 
In comparison with the PCCR of the whole system, we see that the m-PCCR rapidly decays to a small value on a time scale determined by $\hat{\gamma}(\theta_0)$, before becoming worse on the time scale where $\pr_{\theta_0}[t<\tau]$ begins to increase. 

\begin{rmk}
\label{rmk:ChatC}
In Figure \ref{fig:contractionbounds}~we have plotted $H(t)$ and $\hat{H}(t)$ over time without the constants $C,\,\hat{C}$ (that is, with the assumption that $C\simeq\hat{C}$). 
While the constants $C,\hat{C}$, can in principle be computed explicitly by tracing the proofs of Theorems \ref{thm:ergodic_diffusion}~\&~\ref{thm:quasiergodic_diffusion}~and all results leading to these theorems, doing so would clutter the present discussion. 
Instead, we formally justify the assumption that $C$ and $\hat{C}$ are comparable by noting that these constants are determined by five quantities, these being: 
\begin{itemize}
\item The constant $c$ appearing in Lemma \ref{lemma:mtconvergence2}, which only depends on $\pi_0$ and is therefore the same for \eqref{eq:doublewell}~\&~\eqref{eq:cutoffwell}, 
\item The constants $\hat{\underline{m}}$ and $\underline{m}$, which are approximately equal under a suitable choice of prior distribution, 
\item The constant $k(\theta)$ appearing in \eqref{eq:TV}, and its analogue $\hat{k}(\theta)$ for \eqref{eq:cutoffwell}, which are determined as in Appendix \ref{sec:spectral}~and can be assumed to be the same order of magnitude so long as the initial distributions of \eqref{eq:doublewell}~and \eqref{eq:cutoffwell}~are similar, 
\item $I_{\theta_0}$ and the extrema of $\partial_\theta V_{\theta_0}$, via the constants$\norm*{g_1}_\infty$ and $\norm*{g_2}_\infty$ appearing in the proof of Proposition \ref{thm:ergodicdecay}~-- note that for \eqref{eq:doublewell}~\&~\eqref{eq:cutoffwell}, $I_{\theta_0}^{-1/2}=2=\hat{I}_{\theta_0}^{-1/2}$, and at their extremal points, $\partial_\theta V_{\theta_0}$ and $\partial_\theta\hat{V}_{\theta_0}$ are identical. 
\item $\norm*{V_{\theta_0}}_{C^3}$ and $\norm{\hat{V}_{\theta_0}}_{C^3}$, which are equal by construction. 
\end{itemize}
Consequently, $C$ and $\hat{C}$ are at least the same order of magnitude. 
This loose argument justifies our decision to compare $H(t)$ and $\hat{H}(t)$ by the quantities appearing within the brackets of \eqref{eq:PCCR}~\&~\eqref{eq:mPCCR}, omitting the constants $C,\,\hat{C}$. 
\end{rmk}

While $\hat{s}<s$ in this example, so that the contraction of the posterior distribution constructed from the modified model $\hat{\pi}_t$ is in some sense somewhat slower than the contraction of the original posterior distribution, the fact that $\hat{\gamma}(\theta_0)\gg\gamma(\theta_0)$ means that the \emph{certainty}~of this contraction is greatly improved by studying the posterior distribution corresponding to \eqref{eq:cutoffwell}~on a time scale of the order $10^{11}$. 

\begin{figure}[h!]
\centering
\begin{subfigure}{0.5\textwidth}
\centering
\includegraphics[scale=0.2]{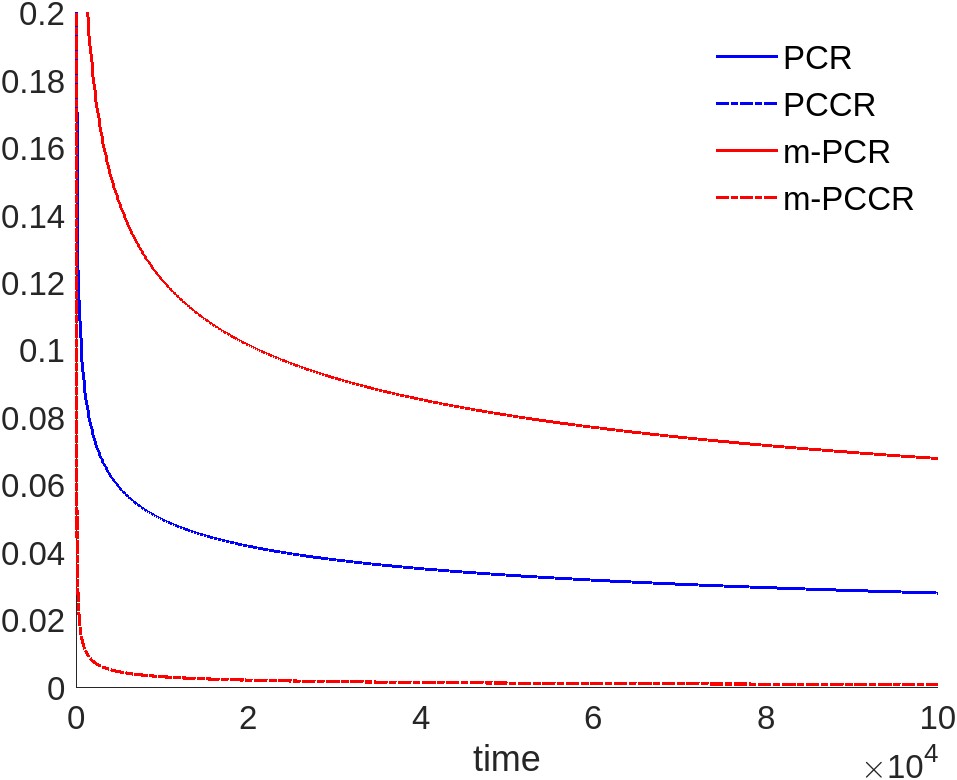}
\caption{}
\end{subfigure}%
~ 
\begin{subfigure}{0.5\textwidth}
\centering
\includegraphics[scale=0.2]{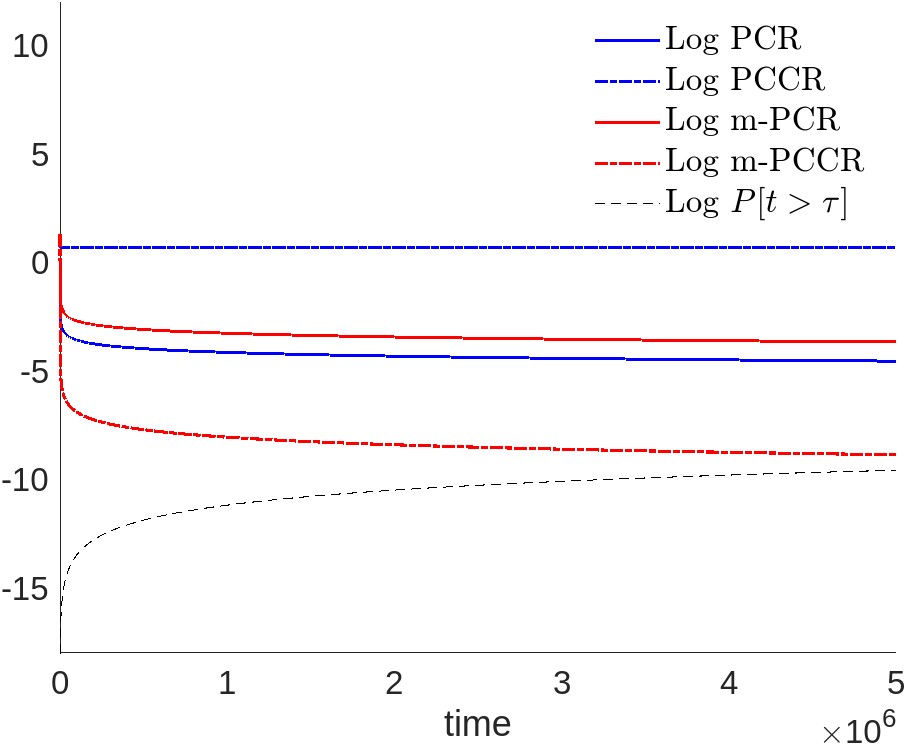}
\caption{}
\end{subfigure}
\caption{
(a) Bounds for the PCR and PCCR, for the original and modified system, with $\sigma=0.1$. 
The bound on the PCCR for the original system is above one up to very large times, due to the smallness of $\gamma(\theta_0)$ and the appearance of more than one term in \eqref{eq:PCCR}~which is not small at small times. 
(b) Logarithmic plot of the same data as presented in (a), on a much longer time scale, to illustrate how poor the certainty bound on the rate of contraction remains for very long times compared to that of the posterior distribution constructed for the modified model \eqref{eq:cutoffwell}. 
The probability of $\{t>\tau\}$ is also plotted, to demonstrate that the m-PCCR estimated via the modified system \eqref{eq:cutoffwell}~remains small over very long time scales, as per equation \eqref{eq:hatPP}. 
The code used to generate these figures is available upon request. 
}
\label{fig:contractionbounds}
\end{figure}

\begin{figure}[h!]
\centering
\begin{subfigure}{0.5\textwidth}
\centering
\includegraphics[scale=0.2]{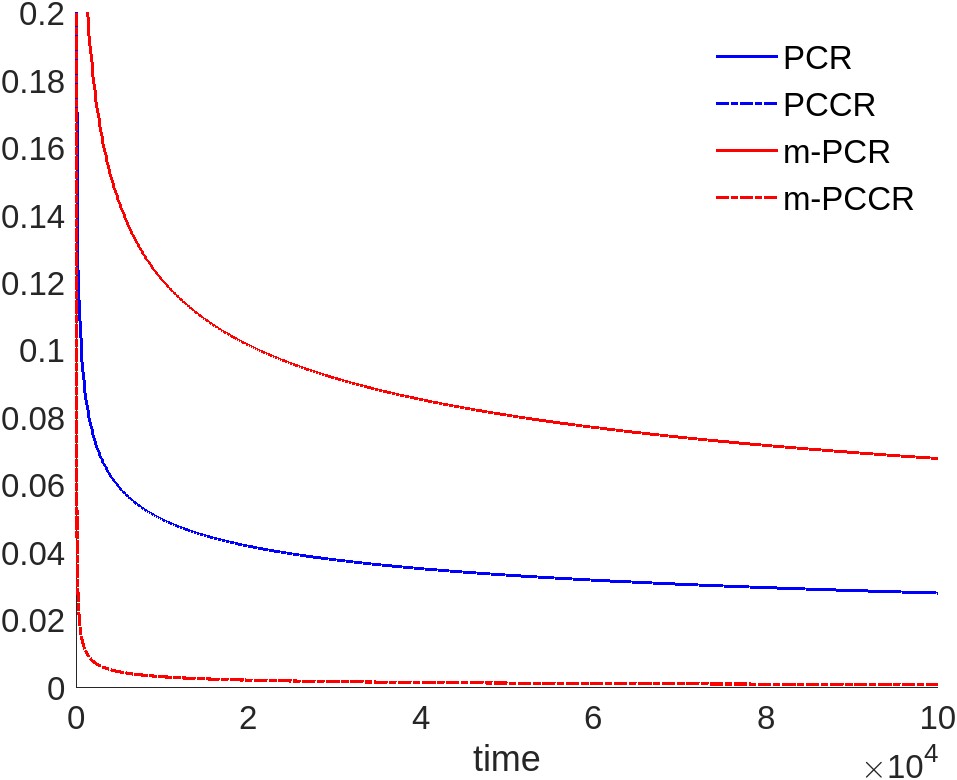}
\caption{}
\end{subfigure}%
~ 
\begin{subfigure}{0.5\textwidth}
\centering
\includegraphics[scale=0.2]{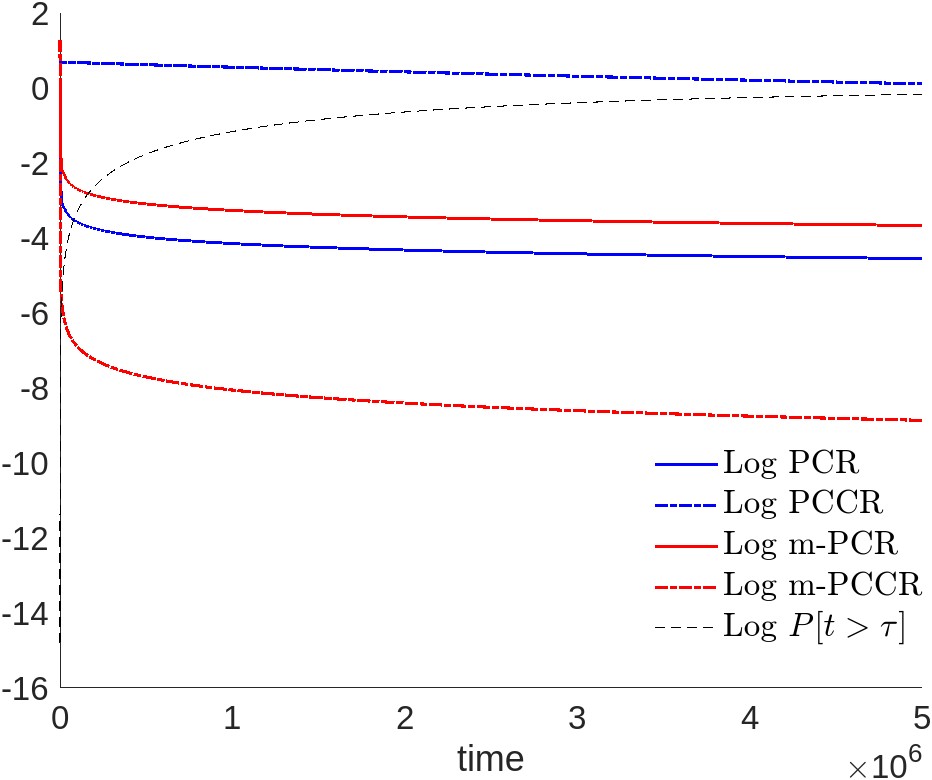}
\caption{}
\end{subfigure}
\caption{
(a) Bounds for the contraction rates and rate of contraction certainty, for the original and modified system, now with $\sigma=0.13$. 
At this noise strength, the bounds on the true posterior distribution being reflected by the posterior distribution constructed from the modified system are still very small on some time scale, but begin to increase after a finite time interval. 
The values of $s$ and $\hat{s}$ and the asymptotic value of $\hat{\gamma}(\theta_0)$ are the same as in Figure \ref{fig:contractionbounds}, while at this noise strength $\lambda(\theta_0)=\gamma(\theta_0)\simeq(3.763\ldots)10^{-7}$. 
In this case, we clearly see $\pr_{\theta_0}[t>\tau]$ begin to increase after a finite time interval, indicating that the m-PCR and m-PCCR begin to fail to give us useful information on the true system after some time. 
Code used to generate these figures is available upon request. 
}
\end{figure}

\begin{rmk}
We note that our analysis for the modified system governed by \eqref{eq:cutoffwell}~would also hold if we replaced the double well in \eqref{eq:doublewell}~by a \emph{cutoff}~double well (for instance, such that $U_{\theta_0}(x)$ tended to a positive constant as $\abs*{x}\rightarrow\infty$), so that \eqref{eq:doublewell}~possessed no invariant probability measure. 
In this case, to the authors' knowledge, there is no standard theory guaranteeing posterior consistency of a Bayesian inference scheme for such a system. 
Nevertheless, the \emph{metaconsistency}~guaranteed by our analysis of the system \eqref{eq:cutoffwell}~and Theorem \ref{thm:quasiposteriorconsistency}~still holds. 
\end{rmk}

\subsection{Learning a diffusion in a degenerate double well}
\label{sec:degeneratedoublewell}
Let us now suppose that we try to fit a time series onto a family of It{\^o}~diffusions of the form \eqref{eq:potentialSDE}~with gradient drift defined by a double well potential of the form 
\begin{equation}
\label{eq:degeneratedoublewell}
U_\theta(x)\,=\,x^4-\theta x^2. 
\end{equation}
In this case, with $\theta_0=4$, we have 
\[
\ExpOp_{\theta_0}\left[\partial_\theta U_{\theta_0}'(\xi)\right]\,=\,\ExpOp_{\theta_0}\left[-2\xi\right]\,=\,0, 
\]
where $\xi$ is a random variable distributed as $\mu_{\theta_0}$ under $\pr_{\theta_0}$, and the above quantity equals zero by the symmetry of $\mu_{\theta_0}$. 
Hence, the estimates on the rate of contraction of a posterior distribution in any Bayesian setup involving this family of models do not apply in this scenario, due to a problem of non-identifiability of the parameter $\theta$. 

However, if we replace the family of potentials in \eqref{eq:degeneratedoublewell}~by a $C^2$ cutoff family of potentials, for instance of the form in \eqref{eq:C2cutoff}, this solves the problem of non-identifiability. 
Indeed, we again find $\hat{s}=0.827\ldots$, computed numerically as in the non-degenerate example of Section \ref{sec:doublewell}, and that $\hat{\gamma}(\theta_0)\simeq2.029\ldots$. 
So, while we have no guarantee of contraction of the true posterior distribution of a Bayesian setup for It{\^o}~diffusions with gradient drift given by the potential \eqref{eq:degeneratedoublewell}, we find that the posterior distribution of the modified system \eqref{eq:cutoffwell}~converges at a relatively rapid rate.

\section{Metastable decomopositions of parameter space}
\label{sec:decomp}
In the sections above, we have seen that we may overcome the slow-down of a Bayesian posterior caused by the small spectral gap of a metastable diffusion by restricting attention to exactly one of its metastable states. 
However, there is the possibility that the Fisher information of the system restricted to a single metastable state is much less than the Fisher information of the whole system. 
But, if the Fisher information in each metastable state is somehow independent of that in the other states, in a sense to be made precise, then it seems natural that we might be able to speed up learning by focusing attention on each metastable state, marginalizing the ``slow to learn'' directions in parameter space, and then tensoring the ``sub-posteriors'' constructed from each metastable subsystem together. 
In this way, as we demonstrate, one overcomes the slow down in Bayesian inference caused by the small spectral gap of the whole system and the potentially small Fisher information of the metastable subsystems. 
To make this all precise, we introduce the following assumption. 
For the sake of not overburdening our notation, we focus on the setting of systems with precisely two metastable states, though the same arguments generalize to the case of $n\in\N$ metastable states, suitable modifications being made. 

\begin{ass}
\label{ass:cutandsew}
Let the SDE \eqref{eq:SDE}~satisfy Assumption \ref{ass:fisher}. 
Let $O_1,O_2\subset\R^d$ be smooth bounded domains such that $\mu_{\theta_0}(O_1\cup O_2)\ge1-\epsilon/3$ for some small, fixed $\epsilon>0$, and let $\hat{V}_\theta^1,\,\hat{V}_\theta^2$, be vector fields on $\R^d$ such that each pair $(O_i,\hat{V}_\theta^i)$ satisfies Assumption \ref{ass:metafisher}. 
Denote the ergodic measure of 
\begin{equation}
\label{eq:Xi}
d\hat{X}\,=\,\hat{V}_\theta^i(\hat{X})\,dt + \sigma\,dW_t 
\end{equation}
by $\mu^i_\theta$, and let $\xi^i$ be a random variable on $\R^d$ distributed as $\mu^i_{\theta_0}$. 
Assume that 
\begin{equation}
\label{eq:muepsilon}
\mu_{\theta_0}^i(O_i)\,\ge\,1-\epsilon/3, 
\end{equation}
which can be achieved for arbitrary $\epsilon>0$ by taking $\hat{V}_{\theta_0}^i$ such that $O_i$ attracts the trajectories of \eqref{eq:Xi}~with $\sigma=0$ at a sufficiently fast rate. 
Denote the singular values of the matrices 
\[
\ExpOp_{\theta_0}\left[\partial_\theta\hat{V}_{\theta_0}^1(\xi^1)\right],\qquad  
\ExpOp_{\theta_0}\left[\partial_\theta\hat{V}_{\theta_0}^2(\xi^2)\right], 
\]
by $s_1^1\le\cdots\le s_p^1$ and $s_1^2\ge\cdots\ge s_p^2$, and let $S_j^i$ be the generalized eigenspace associated with $s_j^i$ 
(note that the $s_i^1$ are labeled in ascending order, while $s_i^2$ are labeled in descending order). 
Assume that for some $1\le k<\ell\le p$ we have 
\[
\begin{aligned}
&s_i^1\,\ll\,s_i^2\quad\text{ for }\,i\in\{1,\ldots,k\}, \qquad\qquad
s_i^1\,\gg\,s_i^2\quad\text{ for }\,i\in\{\ell,\ldots,p\},  
\end{aligned}
\]
and that there exists a decomposition $\Theta=\Theta_0\oplus\Theta_1\oplus\Theta_2$, 
such that each $\Theta_i$ is linear, and 
\[
\Theta_1\,=\,\bigcup_{j\ge\ell}S_j^1,\qquad\Theta_2\,=\,\bigcup_{j\le k}S_j^2.  
\]
\end{ass}

The assumption $\mu_{\theta_0}(O_1\cup O_2)\ge1-\epsilon/3$ is natural when \eqref{eq:SDE}~is ergodic with precisely two metastable states contained in $O_1$ and $O_2$. 
See for instance Figure \ref{fig:density}. 

Before proceeding, let us consider the following scenario: 
Suppose that each $V_\theta$ is the gradient of a double well potential, $V_\theta=\nabla U_\theta$, and that $\Theta=\Theta_1\times\Theta_2\simeq\R^2$ (so, in this example, $\Theta_0=\varnothing$). 
Suppose moreover that $\Theta_1,\Theta_2$, each model the depths of one of the wells in $U_\theta$, but that changing the depth of one well has little or no effect on the shape or depth of the other. 
Then, restricting attention to the first well of $U_\theta$ may speed up the PCCR due to the substantial increase in the spectral gap (as follows from Theorem \ref{thm:quasiposteriorconsistency}~and the discussion of Appendix \ref{sec:spectral}), but will slow down the PCR along the second parameter axis $\Theta_2$. 
This is due to the fact that we can detect very little about the second parameter from observing dynamics in the first well alone. 

Our goal here is to establish how, under Assumption \ref{ass:cutandsew}, one can construct a posterior distribution $\tilde{\pi}_t$ that has all of the advantages of learning the dynamics of each well individually, and none of the disadvantages caused by only being able to see one of the axes $\Theta_i$ in each well. 
For time series $\hat{X}^i$ obtained from \eqref{eq:Xi}, $i\in\{1,2\}$, we construct posteriors $\hat{\pi}_t^i(\,\cdot\mid \hat{X}^i\,)$ on $\Theta$ as before. 
For $i\in\{1,2\}$, we then marginalize out $\Theta_j$, where $j\neq i$, 
\[
\begin{aligned}
\tilde{\pi}^1_t(F|\,X^1)\,&\coloneqq\,\int_{F\cup\Theta_2}\hat{\pi}_t^1(\theta|\,\hat{X}^1)\,\pi_0(d\theta), \qquad\quad\,\text{ and } \\
\tilde{\pi}^2_t(F|X^2)\,&\coloneqq\,\int_{F\cup\Theta_0\cup\Theta_1}\hat{\pi}_t^2(\theta|\hat{X}^2)\,\pi_0(d\theta)\qquad\text{ for $F\subset\Theta$. } 
\end{aligned}
\]
We note that $\Theta_0$ could be included in the definition of $\tilde{\pi}^2_t(\,\cdot\,|\hat{X}^2)$ instead of $\tilde{\pi}^1_t(\,\cdot\,|\hat{X}^1)$ without changing the present discussion. 
Then, our construction of $\tilde{\pi}_t$ is completed by taking the product 
\begin{equation}
\label{eq:tildepi}
\tilde{\pi}_t(F\,|\,\hat{X}^1,\hat{X}^2)\,=\,\tilde{\pi}_t^1(F|\hat{X}^1)\otimes\tilde{\pi}_t^2(F|\hat{X}^2). 
\end{equation}

In addition to such an ``annealed'' posterior, we need an ``annealed'' Fisher information matrix to facilitate our analysis. 
First, for $i\in\{1,2\}$, let $\hat{I}_{\theta_0}^i$ denote the Fisher information matrix associated with \eqref{eq:Xi}, as in \eqref{eq:I}. 
We then let $\tilde{I}_{\theta_0}$ denote the unique matrix such that 
\[
\tilde{I}_{\theta_0}u\,=\,
\begin{cases}
\hat{I}_{\theta_0}^1u&\text{ if $u\in\Theta_0\cup\Theta_1$},\\
\hat{I}_{\theta_0}^2u&\text{ if $u\in\Theta_2$}. 
\end{cases}
\]

Our next major result establishes that the PCR and PCCR of $\tilde{\pi}_t$ are \emph{both}~much faster than those of $\pi_t$ in the setting of Assumption \ref{ass:cutandsew}. 
Essentially, this result indicates that, from a metastable decomposition of the observed dynamics, one has a linear decomposition of the parameter space over which models are fit, in accordance with the singular values of the Fisher information matrix of the process restricted to each of its metastable states. 

\begin{thm}
\label{thm:cutandsew}
Let Assumption \ref{ass:cutandsew}~hold, and let $\{s_1,\ldots,s_p\}$ denote the singular values of $\ExpOp_{\theta_0}\left[\partial_\theta V_{\theta_0}(\xi)\right]$. 
Let $\tilde{\pi}_t$ be as defined in \eqref{eq:tildepi}. 
Then, 
\[
\tilde{s}\,\coloneqq\,\max\left\{\min\left\{s_1^2,\ldots,s_k^2\right\},\min\left\{s_{k+1}^1,\ldots,s_p^1\right\}\right\} \,\gg\,s_1, 
\]
and defining $\tilde{\epsilon}_t\coloneqq\delta\tilde{s}^{-\alpha}t^{-\alpha/2}$ for some fixed $\delta>0$, $\alpha\in(0,1)$, along with
\[
\tilde{F}_t\,\coloneqq\,B_{\tilde{\epsilon}_t}(\theta_0),\qquad \tilde{U}_t\,\coloneqq\,t^{1/2}\tilde{I}_{\theta_0}^{1/2}B_{\tilde{\epsilon}_t}(0), 
\]
we have 
\begin{equation}
\label{eq:tildeconsistency}
\pr_{\theta_0}\left[\tilde{\pi}_t(\tilde{F}_t|\tilde{X}^1,\tilde{X}^2)<1-\tilde{H}_t\right]\,\le\,\tilde{H}_t, 
\end{equation}
where $\tilde{H}_t\coloneqq\max\left\{\hat{H}_t^1,\hat{H}_t^2\right\}$, 
with $\hat{H}_t^i$ being defined for \eqref{eq:Xi}~as in Theorem \ref{thm:quasiposteriorconsistency}. 
\end{thm}

In essence, as $s_1$ is the PCR of the original system, we see that, in the setting of Theorem \ref{thm:cutandsew}, learning each subsystem independently, and only after this putting together what has been learned, can substantially improve the PCR. 
Together with the concrete results on spectral gaps discussed in the previous sections, we see that learning metastable substrates independently, then annealing what is learned back together, can speed up both the PCR \emph{and}~PCCR. 
Of course, in practical terms, this result is only useful to an experimentalist if the initial conditions of an experiment can be controlled, such that it is known when a time series begins in a given metastable state. 

\begin{proof}[Proof of Theorem \ref{thm:cutandsew}]
Let $\norm*{A}_j$ denote the norm of a matrix $A$ defined by taking the $j$th singular value of $A$. 
Note that, by \eqref{eq:muepsilon}, 
\[
\ExpOp_{\theta_0}\left[\partial_\theta V_{\theta_0}(\xi)\right]\,=\,\ExpOp_{\theta_0}\left[\partial_\theta \tilde{V}^1_{\theta_0}(\xi^1)\right] + \ExpOp_{\theta_0}\left[\partial_\theta \tilde{V}^2_{\theta_0}(\xi^2)\right] + 3\epsilon.
\]
Consequently, 
\[
\norm*{\ExpOp_{\theta_0}\left[\partial_\theta V_{\theta_0}(\xi)\right]}_j-\norm*{\ExpOp_{\theta_0}\left[\partial_\theta \tilde{V}^1_{\theta_0}(\xi^1)\right]}_j\,\le\,3\epsilon + \norm*{\ExpOp_{\theta_0}\left[\partial_\theta \tilde{V}^2_{\theta_0}(\xi^2)\right]}_j, 
\]
so that when $s_j^2$ is small, $s_j$ must be close to $s_{p-j}^1$. 
By the same argument, when $s_{p-j}^1$ is small, $s_j$ must be close to $s_j^2$. 
It follows that $s_1\simeq s_1^1$ or $s_1\simeq s_p^2$. 
By Assumption \ref{ass:cutandsew}, we then have that 
\[
s_1^1,s_p^2\,\ll\,\tilde{s}. 
\]
To prove \eqref{eq:tildeconsistency}, remark that Theorem \ref{thm:ergodic_diffusion}~applies to each of $\tilde{\pi}^1_t(\,\cdot\,|X^1),\,\tilde{\pi}^2_t(\,\cdot\,|X^2)$, and we may take the product of these measures to conclude the result on the whole parameter space. 
\end{proof}

\section{Discussion}
\label{sec:Discussion}
The Bayesian PCR of a statistical model consisting of ergodic diffusions can be precisely characterized in terms of the Fisher information of the model, while the PCCR is characterized as a combination of the Fisher information and spectral gaps of the model. 
In the case of metastable diffusions, the spectral gap is very small, leading to a slow PCCR, which may limit the ability to make inferences from observations of metastable time series due to the amount of data needed. 
This slow down in the PCCR for metastable diffusions may be overcome by limiting attention to a single metastable state, as described in Section \ref{sec:ergodicdiffusions}~and demonstrated for a diffusion in a double well potential in Section \ref{sec:doublewell}. 
Restricting attention to a single metastable subsystem may have the added benefit of increasing the Fisher information, in particular in the non-identifiable setting, where the Fisher information may be zero due to a degenerate symmetry in the modeled system. 
For instance, in Section \ref{sec:degeneratedoublewell}, we see how restricting attention to one metastable state of a metastable system can make a non-identifiable model identifiable by removing such symmetry. 
On the other hand, in the case where each metastable state of a system is only affected by changes in a subspace of parameter space, restricting to a single metastable state may slow down the PCR by leading to a small Fisher information. 
To circumvent this, in Section \ref{sec:decomp}~we have seen how one may further restrict the posterior distribution constructed from each metastable subsystem to the directions in parameter space which the metastable subsystem can see, and then take the product of the marginalized ``sub'' posteriors. 
Doing so allows us to overcome possible slowdowns in both the PCR and PCCR of metastable systems caused by a small spectral gap of the whole system, and a small Fisher information of the sub-systems. 
Future work may focus on algorithmic implementations of this theory, as well as applications to the foundations of time series inference and machine learning. 
It would also be of interest to understand how metastability of the statistical models would influence decompositions of the parameter space $\Theta$ when $\Theta$ is fundamentally nonlinear, as in the nonparametric setting where $\Theta$ is infinite dimensional.

\section*{Acknowledgements}
ZPA acknowledges Germany's Excellence Strategy, via the Berlin Mathematics Research Center MATH+ project EF45-5, as well as the Center for Scalable Data Analytics and Artificial Intelligence ScaDS.AI-Leipzig, for supporting this research. 
ZPA would also like to thank the Max Planck Institute for Mathematics in the Sciences for hosting him during this research, and Maximilian Engel, for his guidance and many helpful discussions on the topic of metastability. SM would like to acknowledge partial funding from NSF DBI 1661386, NSF IIS 15-46331, as well as high-performance computing partially supported by grant 2016- IDG-1013
from the North Carolina Biotechnology Center as well as the Alexander von
Humboldt Foundation, the BMBF and the Saxony State Ministry for Science.

\appendix
\section{Some notes on spectral theory}
\label{sec:spectral}
In this appendix, we discuss some results on the spectral properties of the Markov generator of an SDE with gradient drift and additive noise, 
\begin{equation}
\label{eq:potentialSDE}
dX_t\,=\,-\nabla U(X_t)\,dt + \sigma\,dW_t, 
\end{equation}
over an ambient probability space $(\Omega,\F,\pr)$. 
These spectral properties are then connected to the metastable nature of \eqref{eq:potentialSDE}. 
To the authors' knowledge, the quantitative theory of metastable states, encompassing apparent convergence rates to metastable states and transition times, is only complete in the case where $U$ is a Morse function. 
For this reason we restrict our attention to this case in this appendix, though we expect the principle of our discussion to translate to metastable states of greater complexity. 
For recent work on transition times between metastable states of greater complexity, see \cite{M23,VW24}.

Before proceeding, we fix some notation. 
For a negative semi-definite essentially self adjoint operator $\mathcal{A}:D(\mathcal{A})\subset H\rightarrow H$ on a separable Hilbert space $H$, let $\spec(-\mathcal{A})$ denote the spectrum of $-\mathcal{A}$. 
Let $\lambda(\mathcal{A})\coloneqq\inf\spec(-\mathcal{A})$ denote the bottom eigenvalue of the spectrum of $-\mathcal{A}$, and let $\gamma(\mathcal{A})\coloneqq\inf\spec(-\mathcal{A})\backslash\{\lambda(\mathcal{A})\}$ be the next closest point in the spectrum of $-\mathcal{A}$. 
Let $\kappa(\mathcal{A})\coloneqq\gamma(\mathcal{A})-\lambda(\mathcal{A})$ denote the spectral gap of $\mathcal{A}$. 

\begin{ass}
\label{ass:kzeros}
The system \eqref{eq:potentialSDE}~has global in time solutions, and is uniformly ergodic with unique ergodic measure $\mu$. 
Moreover, \eqref{eq:potentialSDE}~has initial distribution $\nu=h\mu$ with $h\in L^2(\R^d,\mu)$, and 
\begin{enumerate}
\item 
$U$ is a non-negative $C^3$ Morse function satisfying the growth condition \cite[Assumption 1.5]{MS14}. 
\item 
The potential $U$ has exactly $n\in\N$ local minima, labeled $x_1,\ldots,x_n,$ and for $k\in\{1,\ldots,n\}$ the Hessian $H_k\coloneqq \nabla^2U(x_k)$ is strictly positive definite. 
\item 
Between any two minima $x_i,x_j$, the saddle height between $x_i$ and $x_j$ is attained at a unique critical point of $U$, denoted $s_{i,j}\in\R^d$. 
\item 
For each pair $i,j$, $\tilde{H}_{i,j}\coloneqq\nabla^2U(s_{i,j})$ has precisely one negative eigenvalue, denoted $\lambda_{i,j}$. 
\item 
The minima $\{x_1,\ldots,x_n\}$ are ordered such that $x_1$ is a global minimum of $U$, and for some $\eta>0$, 
\[
U(s_{1,2})-U(x_2)\,\ge\,U(s_{1,i})-U(x_i)+\eta,\qquad i\in\{3,\ldots,n\}.  
\]
\end{enumerate}
\end{ass}

Under Assumption \ref{ass:kzeros}, the Markov generator $\L$ of \eqref{eq:potentialSDE}~exists in the $L^2_{\mu}$-topology, and its action on $f\in C^2(\R^d)\cap L^2(\R^d,\mu)$ can be expressed as  
\begin{equation}
\label{eq:L}
\L f\,=\,\frac{\sigma^2}{2}\Delta f -\nabla U\cdot\nabla f, 
\end{equation}
In this appendix, we collect rigorous results on the spectral properties of $\L$, and its restriction to neighbourhoods of the $x_k$ with Dirichlet boundary conditions. 
These results justify our estimates on the constant $\gamma(\theta_0)$ appearing in Assumption \ref{ass:uniformergodic}~in the case of metastable diffusions, and on the rate of escape via estimates on the distribution of $\pr_{\theta_0}\left[t<\tau\right]$, as appears in Theorem \ref{thm:quasiposteriorconsistency}. 

\begin{thm}
\label{thm:uniformergodic}
Let Assumption \ref{ass:kzeros}~hold. 
Then, $\lambda(\L)=0$ is a simple eigenvalue of $\L$, and 
\begin{equation}
\label{eq:gammaL}
\gamma(\L)\,\simeq\,\left(\frac{2\abs*{\lambda_{1,2}}\sqrt{\det\nabla^2 H_2}}{\pi\sqrt{\abs*{\det\nabla^2\tilde{H}_{1,2}}}}\right)\exp\left(2\frac{U(x_2)-U(s_{1,2})}{\sigma^2}\right), 
\end{equation}
equality holding up to a multiplicative constant of order $1+\sqrt{\frac{\sigma^2}{2}}\abs*{\ln\frac{\sigma^2}{2}}^{3/2}$. 
Moreover, 
\begin{equation}
\label{eq:appTV}
\norm*{\pr\left[X_t\in\,\cdot\,\right]-\mu(\,\cdot\,)}_{TV}\,\le\,e^{-\gamma(\L)t/2}\norm*{h-1}_{L^2_\mu}. 
\end{equation}
\begin{proof}
By \cite{MS14}, \eqref{eq:potentialSDE}~has a spectral gap $\gamma(\L)$ such that \eqref{eq:gammaL}~holds up to a multiplicative constant of order $1+\sqrt{\frac{\sigma^2}{2}}\abs*{\ln\frac{\sigma^2}{2}}^{3/2}$ (we note that \cite{MS14}~requires $U\ge0$, but this translates to our setting by simply adding a constant to $U$). 
Equation \eqref{eq:appTV}~then follows from \cite[Theorem 2.1]{CG09}. 
\end{proof}
\end{thm}

As suggested by the discussion of Section \ref{sec:ergodicdiffusions}~on ergodic diffusions, to study the metastable behaviour of \eqref{eq:potentialSDE}~near each $x_k$ we consider a modified SDE with a single well at $x_k$. 
Specifically, fix $k\in\{1,\ldots,n\}$ and $\delta>0$ such that $B_\delta(x_k)$ is strictly contained in the basin of attraction of $x_k$, and fix a smooth potential $\hat{U}_\delta$ satisfying 
\begin{enumerate}
\item $\hat{U}_\delta=U$ in the ball $B_\delta(x_k)$, and  
\item $\hat{U}_\delta$ is uniformly convex, so there exists $u_2>0$ such that $\nabla^2\hat{U}_\delta(x)\ge u_21$ for all $x\in\R^d$. 
\end{enumerate}
Then, consider 
\begin{equation}
\label{eq:newpotentialSDE}
dX_t\,=\,\nabla\hat{U}_\delta(X_t)\,dt + \sigma\,dW_t, 
\end{equation}
and its associated Markov generator $\hat{\L}_\delta$ defined in the topology of $L^2(\R^d,\hat{\mu}_\delta)$, where $\hat{\mu}_\delta$ is the unique ergodic measure of \eqref{eq:newpotentialSDE}~(guaranteed to exist by Assumption \ref{ass:kzeros}~and the construction of $\hat{U}_\delta$). 
The generator $\hat{\L}_\delta$ acts on $f\in C^2(\R^d)\cap L^2(\R^d,\hat{\mu}_\delta)$ as 
\begin{equation}
\hat{\L}_\delta f\,=\,\frac{\sigma^2}{2}\Delta f+ \nabla\hat{U}_\delta\cdot\nabla f. 
\end{equation}

\begin{thm}
\label{thm:modergodic}
Let Assumption \ref{ass:kzeros}~hold. 
Then, \eqref{eq:newpotentialSDE}~has a global in time solution $(\hat{X}_t)_{t\ge0}$ and a unique ergodic measure $\hat{\mu}$. 
Moreover, 
\begin{equation}
\label{eq:appTV2}
\norm*{\pr[\hat{X}_t\in\,\cdot\,]-\hat{\mu}(\,\cdot\,)}_{TV}\,\le\,e^{-\hat{\gamma}t}, 
\end{equation}
where for all $\sigma>0$, $\hat{\gamma}=u_2$. 
\begin{proof}
By the Bakry-Emery criterion, \eqref{eq:newpotentialSDE}~satisfies a Poincar{\'e}~inequality with constant $1/u_2$, see \cite[Proposition 4.8.1]{BGL}. 
The contraction \eqref{eq:appTV2}~then follows from \cite[Theorem 2.1]{CG09}.  
\end{proof}
\end{thm}

Now, with $k\in\{1,\ldots,n\}$ and $\delta>0$ as above, define 
\[
\tau_\delta\,\coloneqq\,\inf\left\{t>0\,:\,X_t\in\partial B_\delta(x_k)\right\}, 
\]
and for $f\in L^2(B_\delta(x_k),\mu)$ set 
\[
P_t^\delta f(x)\,\coloneqq\,\ExpOp\left[f(X_t)1_{\{t<\tau_\delta\}}\right]. 
\]
Just as for the Markov semigroup $(P_t)_{t\ge0}$, the sub-Markov semigroup $(P_t^\delta)_{t\ge0}$ is strongly continuous in the topology of $L^2(B_\delta(x_k),\mu)$. 
It therefore has a strong generator, which we denote by $\L_\delta$, and which for $f\in C^2_0(B_\delta(x_k))\cap L^2(B_\delta(x_k),\mu)$ can be expressed as 
\[
\L_\delta f(x)\,=\,\beta\Delta f(x) -\nabla U(x)\cdot\nabla f(x). 
\]
That is, $\L_\delta$ has the same form as $\L$, but with Dirichlet boundary conditions on $B_\delta(x_k)$. 
Now, we have the following result on the escape time of $(X_t)_{t\ge0}$ from the well $B_\delta(x_k)$. 

\begin{thm}
\label{thm:escaperate}
$\lambda(\L_\delta)$ is a strictly negative simple eigenvalue of $\L_\delta$, and the spectral gap $\kappa(\L_\delta)$ is bounded away from zero uniformly in $\sigma>0$. 
Finally, for a constant $c>0$ independent of $\sigma>0$, up to a multiplicative error term of order $1+\sqrt{\frac{\sigma^2}{2}}\abs*{\ln\frac{\sigma^2}{2}}^{3/2}$ it holds that 
\begin{equation}
\label{eq:escaperate}
\pr\left[t<\tau_\delta\right]\,=\,ce^{-\lambda(\L_\delta) t} + O\left(e^{-\kappa(\L_\delta) t}\right). 
\end{equation}
Moreover, up to a multiplicative error term of order $1+\sqrt{\frac{\sigma^2}{2}}\abs*{\ln\frac{\sigma^2}{2}}^{3/2}$, we have 
\begin{equation}
\lambda(\L_\delta)\,\simeq\,\left(\frac{2\abs*{\lambda_{1,2}}\sqrt{\det\nabla^2 H_2}}{\pi\sqrt{\abs*{\det\nabla^2\tilde{H}_{1,2}}}}\right)\exp\left(2\frac{U(x_2)-U(s_{1,2})}{\sigma^2}\right).  
\end{equation}
\begin{proof}
The expression for $\pr[t<\tau_\delta]$ can be found, for instance, in \cite{CMSM13}. 
The asymptotic behaviour of $\lambda(\L_\delta)$ follows from \cite[equation (4.76)]{BGK05}~together with \eqref{eq:gammaL}. 
\end{proof}
\end{thm}

\bibliographystyle{plain}
\bibliography{bibliography}

\end{document}